\documentclass{article}

\PassOptionsToPackage{numbers, sort&compress}{natbib}
\usepackage[preprint]{neurips_2021}

\usepackage[utf8]{inputenc} 
\usepackage[T1]{fontenc}    
\usepackage{amsmath,amssymb}
\usepackage{hyperref}       
\usepackage{booktabs}       
\usepackage{amsfonts}       
\usepackage{nicefrac}       
\usepackage{microtype}      
\usepackage{xcolor}         
\usepackage{bm}
\usepackage{amsthm}
\usepackage{shortcuts}
\usepackage{graphicx}
\usepackage{subfig}
\usepackage[capitalise]{cleveref}
\Crefname{assumption}{Assumption}{Assumptions}
\usepackage{autonum}
\theoremstyle{plain}
\newtheorem{theorem}{Theorem}

\newtheorem{lemma}{Lemma}

\theoremstyle{definition}
\newtheorem{assumption}{Assumption}

\newtheorem{remark}{Remark}
\newtheorem{example}{Example}

\newcommand{\bO}{\bar{O}}
\newcommand{\bo}{\bar{o}}
\newcommand{\cO}{\mathcal O}
\newcommand{\F}{\mathcal F}
\newcommand{\A}{\mathcal A}
\newcommand{\X}{\mathcal X}
\newcommand{\Y}{\mathcal Y}

\newcommand{\Nb}{N_{[\,]}}

\newcommand{\para}[1]{\textbf{#1}}

\newcommand{\Ind}{\textbf{1}}

\newcommand{\later}[1]{}

\title{Risk Minimization from Adaptively Collected Data: Guarantees for Supervised and Policy Learning}

\author{%
  Aurélien Bibaut\thanks{Alphabetical order} \\
  Netflix
 \And
    Antoine Chambaz\footnotemark[1]\\
    Université Paris Descartes
\And
    Maria Dimakopoulou\footnotemark[1]\\
    Netflix
\And
    Nathan Kallus\footnotemark[1]\\
    Cornell University and Netflix
\And 
    Mark van der Laan\footnotemark[1]\\
    University of California, Berkeley
}

\begin{document}

\maketitle

\begin{abstract}
Empirical risk minimization (ERM) is the workhorse of machine learning, whether for classification and regression or for off-policy policy learning, but its model-agnostic guarantees can fail when we use adaptively collected data, such as the result of running a contextual bandit algorithm.
We study a generic importance sampling weighted ERM algorithm for using adaptively collected data to minimize the average of a loss function over a hypothesis class and provide first-of-their-kind generalization guarantees and fast convergence rates.
Our results are based on a new maximal inequality that carefully leverages the importance sampling structure to obtain rates with the right dependence on the exploration rate in the data.
For regression, we provide fast rates that leverage the strong convexity of squared-error loss.
For policy learning, we provide rate-optimal regret guarantees that close an open gap in the existing literature whenever exploration decays to zero, as is the case for bandit-collected data. 
\later{We also provide guarantees for model selection using an adaptive leave-one-out cross-validation.}
An empirical investigation validates our theory.
\end{abstract}

\section{Introduction}

Adaptive experiments, wherein intervention policies are continually updated as in the case of contextual bandit algorithms, offer benefits in learning efficiency and better outcomes for participants in the experiment. They also make the collected data dependent and complicate standard machine learning approaches for model-agnostic risk minimization, such as empirical risk minimization (ERM). Given a loss function and a hypothesis class, ERM seeks the hypothesis that minimizes the sample average loss. This can be used for regression, classification, and even off-policy policy optimization. An extensive literature has shown that, for independent data, ERM enjoys model-agnostic, best-in-class risk guarantees and even fast rates under certain convexity and/or margin assumptions \citep[\eg][among others]{bartlett2006convexity,bartlett2005local,koltchinskii2006local}. However, these guarantees fail under contextual-bandit-collected data, both because of covariate shift due to using a context-dependent logging policy and because of the policy's data-adaptive evolution as more data are collected. A straightforward and popular approach to deal with the covariate shift is importance sampling (IS) weighting, whereby we weight samples by the inverse of the policy's probability of choosing the observed action. 
Unfortunately, applying standard maximal inequalities for sequentially dependent data to study guarantees of this leads to poor dependence on these weights, and therefore incorrect rates whenever exploration is decaying and the weights diverge to infinity, as happens when collecting data using a contextual bandit algorithm.

In this paper, we provide a thorough theoretical analysis of IS weighted ERM (ISWERM; pronounced ``ice worm'') that yields the correct rates on the convergence of excess risk under decaying exploration. To achieve this, we present a novel localized maximal inequality for IS weighted sequential empirical processes (\cref{sec:maxinequality}) that carefully leverages their IS structure to avoid a bad dependence on the size of IS weights, as compared to applying standard results to an IS weighted process (\cref{remark:isstructure}). We then apply this result to obtain generic slow rates for ISWERM for both Donsker-like and non-Donsker-like entropy conditions, as well as fast rates when a variance bound applies (\cref{sec:iswermguarantees}).
We instantiate these results for regression (\cref{sec:regression}) and for policy learning (\cref{sec:policylearning}), where we can express entropy conditions in terms of the hypothesis class and obtain variance bounds from convexity and margin assumptions. In particular, our results for policy learning close an open gap between existing lower and upper bounds in the literature (\cref{remark:zhan}). 
\later{We also provide results for model selection using a sequential leave-one out cross-validation (\cref{sec:loo}).}
We end with an empirical investigation of ISWERM that sheds light on our theory (\cref{sec:empirics}).

\subsection{Setting}

We consider data consisting of $T$ observations, $\bO_T=(O_1,\dots,O_T)$, where each observation consists of a state, action, and outcome, $O_t=(X_t,A_t,Y_t)\in\cO=\X\times\A\times\Y$. The spaces $\X,\A,\Y$ are general measurable spaces, each endowed with a base measure $\lambda_\X,\lambda_\A,\lambda_\Y$; in particular, actions can be finite or continuous (\eg, $\lambda_\A$ can be counting or Lebesgue).
We assume the data were generated sequentially in a stochastic-contextual-bandit fashion.
Specifically, we assume that the distribution of $\bO_T$ has a 
density $p^{(T)}$ with respect to (wrt) 
$\lambda_\cO^T=(\lambda_\X\times\lambda_\A\times\lambda_\Y)^T$, which can be decomposed as
$$
p\s T(\bo_T)=\prod_{t=1}^Tp_X(x_t)\tilde g_t(a_t\mid x_t,\bo_{t-1})p_Y(y_t\mid x_t,a_t),
$$
where we write $\bo_t=(x_1,a_1,\dots,y_t)$, using lower case for dummy values and upper case for random variables.
We define $g_t(a\mid x)=\tilde g_t(a\mid x,\bO_{t-1})$ so that $g_t$ represents the \emph{random} $\bO_{t-1}$-measurable context-dependent policy that the agent has devised at the beginning of round $t$, which they then proceed to employ when observing $X_t$.

\begin{remark}[Counterfactual interpretation]\label{remark:counterfactual}We can also interpret this data collection from a counterfactual perspective.
At the beginning of each round, $(X_t,\{Y_t(a):a\in\A\})$ is drawn from some stationary (\ie, time-independent) distribution $P^*$, $X_t$ is revealed, and after acting with a non-anticipatory action $A_t$ we observe $Y_t=Y_t(A_t)$. This corresponds to the above with $p_X$ being the marginal of $X_t$ under $P^*$ and $p_Y(\cdot\mid x,a)$ the conditional distribution of $Y_t(a)$ given $X_t=x$.\end{remark}

\subsection{Importance Sampling Weighted Empirical Risk Minimization}

Consider a class of hypotheses $\F$, a loss function $\ell:\F\times\cO\to\Rl$, and some fixed reference $g^*(a\mid x)$, any function, for example, a conditional density.
As we will see in \cref{ex:regression,ex:classification,ex:policy} below we will often simply use $g^*(x\mid a)=1$.
Define the \emph{population reference risk} as
$$
R^*(f)=\E_{p_X\times g^*\times p_Y}[\ell(f,O)]=\int \ell(f,(x,a,y)) p_Y(y\mid x,a)
g^*(a\mid x)
p_X(x)
d\lambda_\cO(x,a,y)
.
$$

We are interested in finding $f$ with low risk $R^*(f)$. We consider doing so using ISWERM, which is ERM where we weight each term by the density ratio between the reference and the policy at time $t$:
\begin{align}
\hat f_T\in\argmin_{f\in\F}\braces{\hat R_T(f)=\frac1T\sum_{t=1}^T\frac{g^*(A_t\mid X_t)}{g_t(A_t\mid X_t)}\ell(f,O_t)}.
\end{align}

\begin{example}[Regression]\label{ex:regression}
Consider $\Y=\Rl$, $\F\subseteq[\X\times\A\to\Rl]$, and $\ell(f,o)=(y-f(x,a))^2$. Then $f$ with small $R^*(f)$ is good at predicting outcomes from context and action.
In particular, for any $g^*$, we have that $\mu(x,a)=\int y p_Y(y\mid x,a)d\lambda_Y(y)$ solves $\mu\in\argmin_{f:\X\times\A\to\Y}R^*(f)$. And, we can write $R^*(f)-R^*(\mu)=\E_{p_X\times g^*}[(f-\mu)^2(X,A)]
$.

Consider the counterfactual interpretation in \cref{remark:counterfactual}. Then $R^*(f)=\int\E_{P^*}[(Y(a)-f(X,a))^2]g^*(x\mid a)d\lambda_\A(a)$. For example, if $\abs{\A}<\infty$, $\lambda_\A$ is the counting measure, and $g^*(x\mid a)=1$, then $R^*(f)=\sum_{a\in\A}\E_{P^*}[(Y(a)-f(X,a))^2]$ is the total counterfactual prediction error.
Alternatively, if $g^*(a\mid x)=\Ind(a=a^*)$ and given some $\mathcal H\subseteq[\X\to\Y]$ we let $\F=\{f_h(x,a)=h(x):h\in\mathcal H\}$, then we have $R^*(f_h)=\E_{P^*}[(Y(a^*)-h(X))^2]$, that is, the regression risk for predicting the counterfactual outcome $Y(a^*)$ from $X$.
\end{example}

\begin{example}[Classification]\label{ex:classification}
In the same setting as \cref{ex:regression}, suppose $\Y=\{\pm1\}$. Then $\mu(x,a)=2p_Y(1\mid x,a)-1$. And, if we restrict $\F\subseteq[\X\times\A\to\{\pm1\}]$, letting $\ell(f,o)=\frac12-\frac12yf(x,a)$ leads to $R^*(f)$ being misclassification rate, an unrestricted minimizer of which is $\op{sign}(\mu(x,a))$. Focusing on misclassification of $\op{sign}(f(x,a))$ for $\F\subseteq[\X\times\A\to\Rl]$, we can also use a classification-calibrated loss \citep{bartlett2006convexity}, such as logistic $\ell(f,o)=\log(1+\exp(-yf(x,a)))$, hinge $\ell(f,x)=(1-yf(x,a))_+$, \etc.
\end{example}

\begin{example}[Policy learning]\label{ex:policy}
Consider $\Y=\Rl$, $\F\subseteq[\X\times\A\to\Rl]$, $g^*(a\mid x)=1$ and $\ell(f,o)=yf(x,a)$. Then $R^*(f)=\int y p_Y(y\mid x,a)d\lambda_Y(y)f(a\mid x)d\lambda_A(a)p_X(x)d\lambda_X(x)=\E_{p_X\times f\times p_Y}[y]$ is the average outcome under a policy $f$. If we interpret outcomes as costs (or, negative rewards), then seeking to minimize $R^*(f)$ means to seek a policy with least risk (or, highest value).

Consider in particular the counterfactual interpretation in \cref{remark:counterfactual} with $\abs{\A}<\infty$. Consider deterministic policies: given $\mathcal H\subseteq[\X\to\A]$, let $\F=\{f_h(x,a)=\Ind(h(x)=a):h\in\mathcal H\}$. Then we have $R^*(f_h)=\E_{P^*}[Y(h(X))]$, that is, the average counterfactual outcome.
\end{example}

\subsection{Related Literature}

\para{Contextual bandits.} A rich literature studies how to design adaptive experiments to optimize regret, simple regret, or the chance of identifying best interventions \citep[see][and biblioraphies therein]{bubeckbook,lattimore2020bandit}. Such adaptive experiments can significantly improve upon randomized trials (aka A/B tests), which is why they are seeing increased use in practice in a variety of settings, from e-commerce to policymaking \citep{atheytrial,kasytrial,kasy2021adaptive,bakshy2018ae,tewari2017ads,kallus2020dynamic,qiang2016dynamic,li2010contextual}.
However, while randomized trials produce iid data, adaptive experiments do not, complicating post-experiment analysis, which motivates our current study.
Many stochastic contextual bandit algorithms (stochastic meaning the context and response models are stationary, as in our setting) need to tackle learning from adaptively collected data to fit regression estimates of mean reward functions, but for the most part this is based on models such as linear \citep{li2010contextual,chu2011contextual,bastani2020online,goldenshluger2013linear} or H\"older class \citep{rigollet2010nonparametric,rigollet2010nonparametric,hu2020smooth}, rather than on doing model-agnostic risk minimization and nonparametric learning with general function classes as we do here.
\citet{foster2020beyond} use generic regression models but require online oracles with guarantees for adversarial sequences of data.
\citet{simchi2020bypassing} use offline least-squares ERM but bypass the issue of adaptivity by using epochs of geometrically growing size in each of which data are collected iid.
Other stochastic contextual bandit algorithms are based on direct policy learning using ERM \citep{dudik2011efficient,agarwal2014taming,bibaut2020,foster2018contextual}; by carefully designing exploration strategies, they obtain good regret rates that are even better than the minimax-optimal guarantees given only the exploration rates, as we obtain (\cref{remark:zhan}).

\para{Inference with adaptive data.} A stream of recent literature tackles how to construct confidence intervals after an adaptive experiment.
While standard estimators like inverse-propensity weighting (IPW) and doubly robust estimation remain unbiased under adaptive data collection, they may no longer be asymptotically normal making inference difficult.
To fix this, \citet{hadad2019confidence} use and generalize a stabilization trick originally developed by \citet{luedtke_vdL2016} for a non-adaptive setting with different inferential challenges.
Their stabilized estimator, however, only works for data collected by non-contextual bandits. \citet{CADR} extend this to a contextual-bandit setting. Our focus is different from these: risk minimization and guarantees rather than inference.

\para{Policy learning with adaptive data.} \citet{zhan2021policy} study policy learning from contextual-bandit data by optimizing a doubly robust policy value estimator stabilized by a deterministic lower bound on IS weights. They provide regret guarantees for this algorithm based on invoking the results of \citet{rakhlin2015online}. However, these guarantees do not match the algorithm-agnostic lower bound they provide whenever the lower bounds on IS weights decay to zero, as they do when data are generated by a bandit algorithm. For example, for an epsilon-greedy bandit algorithm with an exploration rate of $\epsilon_t=t^{-\beta}$, their lower bound on expected regret is $\Omega(T^{-(1-\beta)/2})$ while their upper bound is $O(T^{-(1/2-\beta)})$. We close this gap by providing an upper bound of $O(T^{-(1-\beta)/2})$ for our simpler IS weighted algorithm. See \cref{remark:zhan}. Our results for policy learning also extend to fast rates under margin conditions, non-Donsker-like policy classes, and learning via convex surrogate losses.

\para{IS weighted ERM.} The use of IS weighting to deal with covariate shift, including when induced by a covariate-dependent policy, is standard. For estimation of causal effects from observational data this usually takes the form of inverse propensity weighting \citep{imbens2015causal}. The same is often used for ERM for regression \citep{freedman2008weighting,robins2000marginal,dimakopoulou2017estimation} and for policy learning \citep{kitagawa2018should,zhao2012estimating,swaminathan2015batch}.
When regressions are plugged into causal effect estimators, weighted regression with weights that depend on IS weights minimize the resulting estimation variance over a hypothesis class \citep{farajtabar2018more,kallus2019intrinsically,kallus2020optimal,rubin2008empirical,cao2009improving}.
All of these approaches however have been studied in the independent-data setting where historical logging policies do not depend on the same observed data available for training, guarantees under which is precisely our focus herein.

\para{Sequential maximal inequalities.} 
There are essentially two strands in the literature on maximal inequalities for sequential empirical processes. One expresses bounds in terms of sequential bracketing numbers as introduced by \citet{geer2000empirical}, generalizing of standard bracketing numbers. Another uses sequential covering numbers, introduced by \citet{rakhlin2015sequential}. These are in general not comparable. \citet{foster2018contextual,zhan2021policy} use sequential $L_\infty$ and $L_p$ covering numbers, respectively, to obtain maximal inequalities.
%
\citet[Chapter 8]{geer2000empirical} gives guarantees for ERM  over nonparametric classes of controlled sequential bracketing entropy. However, applying her generic result as-is to IS weighted processes provides bad dependence on the exploration rate in the case of larger-than-Donsker hypothesis classes (see \cref{remark:isstructure}).
We also use sequential bracketing numbers, but we develop a new maximal inequality specially for IS weighted sequential empirical processes, where we use the special structure when truncating the chaining to avoid a bad dependence on the size of the IS weights. Equipped with our new maximal inequality, we obtain first-of-their kind guarantees for ISWERM, including fast rates that have not been before derived in adaptive settings.

\section{A Maximal Inequality for IS Weighted Sequential Empirical Processes}\label{sec:maxinequality}

A key building block for our results is a novel maximal inequality for IS weighted sequential empirical processes.
For any sequence of objects $(x_t)_{t\geq 1}$, we introduce the shorthand $x_{1:T}$ to denote the sequence $(x_t)_{t=1}^T$.
We say that a sequence of random variables $\zeta_{1:T}$ is $\bO_{1:T}$-predictable if, for every $t \in [T]$, $\zeta_t$ is $\bO_{t-1}$-measurable, \ie, is some function of $\bO_{t-1}$.

\para{IS weighted sequential empirical processes.} Let $P_g$ denote the distribution on $\cO$ with density w.r.t. $\lambda_X \times \lambda_A \times \lambda_Y$ given by $p_X\times g\times p_Y$ and let us use the notation $P_g h(o) := \int h(o) dP_{g}(o)$.
Consider a sequence of $\mathcal{F}$-indexed random processes of the form
$\Xi_T := \left\lbrace (\xi_t(f))_{t =1}^T : f \in \mathcal{F} \right\rbrace$
where, for every $f \in \F$, $\xi_{1:T}(f)$ is an $\bO_{1:T}$-predictable sequence of $\mathcal{O} \to \mathbb{R}$ functions. 
The IS-weighted sequential empirical process induced by $\Xi_T$ is the $\F$-indexed random process 
\begin{align}
M_T(f):=&
\frac1T\sum_{t=1}^T\frac{g_t^*(A_t\mid X_t)}{g_t(A_t\mid X_t)}\prns{\xi_t(f)(O_t)-\Eb{\xi_t(f)(O_t) \mid \bO_{t-1}}}\\
&=\frac1T\sum_{t=1}^T(\delta_{O_t}-P_{g_t})\prns{\frac{g^*}{g_t}\xi_t(f)}.
\end{align}

\para{Sequential bracketing entropy.} 
For any $\bO_{1:T}$-predictable sequence sequence $\zeta_{1:T}$ of functions $\mathcal{O} \to \mathbb{R}$, we introduce the pseudonorm $\rho_{T,g^*}(\zeta_{1:T}) := ( T^{-1}\sum_{t=1}^T \|\zeta_t\|_{2,g^*}^2)^{1/2}$.

We say that a collection of $2N$-many $\bar{O}_{1:T}$-predictable sequences of $\mathcal{O} \to \mathbb{R}$ functions $\{ (\lambda^{k}_{1:T}, \upsilon^{k}_{1:T}) : k\in[N] \}$ is an $(\epsilon, \rho_{T, g^*})$-sequential bracketing of $\Xi_T$, if
    (a) for every $f \in \mathcal{F}$, there exists $k\in[N]=\{1,\dots,N\}$ such that $\lambda_t^k \leq \xi_t(f) \leq \upsilon_t^k~\forall t\in[T]$
    and (b) for every $k \in [N]$, $\rho_{T,g^*}( \upsilon^k_{1:T} - \lambda^k_{1:T}) \leq \epsilon$.
We denote by $\mathcal{N}_{[\,]}(\epsilon, \Xi_T, \rho_{T,g^*})$ the minimal cardinality of an $(\epsilon, \rho_{T,g^*})$-sequential bracketing of $\Xi_T$.

\para{The special case of classes of classes of deterministic functions.} Consider the special case $\xi_t(f) := \xi(f)$, where
$\Xi := \left\lbrace \xi(f) : f \in \mathcal{F} \right\rbrace$ is a class of functions where for every $f \in \mathcal{F}$, $\xi(f)$ 
is a deterministic $\cO \to \Rl$ function.
Observe that for a fixed function $\zeta:\cO \to \Rl$, letting $\zeta_t := \zeta$, 
we have that $\rho_{T,g^*}(\zeta_{1:T}) = \|\zeta\|_{2,g^*}$. Therefore, $\mathcal{N}_{[\,]}(\epsilon, \Xi_T, \rho_{T,g^*})$, the $(\epsilon, \rho_{T,g^*})$-sequential bracketing number of $\Xi_T$, reduces to $N_{[\,]}(\epsilon, \Xi, \|\cdot\|_{2,g^*})$, the usual $\epsilon$-bracketing number $\Xi$ in the $\|\cdot\|_{2,g^*}$ norm.

\para{The maximal inequality.} 
Our maximal inequality will crucially depend on the decay rate of the the IS weights, that is, the exploration rate of the adaptive data collection.
\begin{assumption}\label{asm:exploration}
There exists a deterministic sequence of positive numbers $(\gamma_t)$ such that, for any $t \geq 1$, $\|g^* / g_t\|_\infty \leq \gamma_t$, almost surely. Define $\gamma^\text{avg}_T := T^{-1} \sum_{t=1}^T \gamma_t$ and $\gamma_T^{\max} := \max_{t \in [T]} \gamma_t$.
\end{assumption}
For example, if the data were collected under an $\epsilon_t$-greedy contextual bandit algorithm then we have $\gamma_t=\epsilon_t^{-1}$. If we have $\epsilon_t=t^{-\beta}$ for $\beta\in(0,1)$ then $\gamma_T^{\max} =O(\gamma^\text{avg}_T) =O(T^{\beta})$.

\begin{theorem}\label{thm:max_ineq_IS_weighted_MEP}
Consider $\Xi_T := \{ \xi_{1:t}(f) : f \in \F \}$ as defined above. Suppose that \cref{asm:exploration} holds, and that there exists $B > 0$ such that $\max_{t \in [T]} \sup_{f \in \mathcal{F}} \|\xi_t(f) \|_\infty \leq B$.  
In the special case where $\xi_t(f)=\xi(f)$, $\xi_{1:T}\in\Xi_T$, are deterministic functions,
we let $\widetilde{\gamma}_T := \gamma^\text{avg}_T$. Otherwise, in the general case, we let $\widetilde{\gamma}_T := \gamma_T^{\max}$.
Let $r > 0$. Let $\F_T(r) := \{ f \in \F : \rho_{T,g^*}(\xi_{1:T}(f)) \leq r \}$.
For any $r^- \in [0,r/2]$, and any $x > 0$, it holds with probability at least $1 - 2 e^{-x}$ that
\begin{align}
    \sup_{f \in \mathcal{F}_T(r)} M_T(f) \lesssim &~ r^- + \sqrt{\frac{\widetilde{\gamma}_T}{T}} \int_{r^-}^r \sqrt{\log (1 + \mathcal{N}_{[\,]}(\epsilon, \Xi_T, \rho_{T,g^*}))} d \epsilon
    \\
    &~+ \frac{\gamma_T^{\max} B}{T} \log (1 + \mathcal{N}_{[\,]}(\epsilon, \Xi_T, \rho_{T,g^*})) + \sqrt{\frac{\widetilde{\gamma}_T x}{T}} + \frac{B \gamma_T^{\max} x}{T}.
\end{align}
\end{theorem}

\begin{remark}[Leveraging IS structure]\label{remark:isstructure} 
\Cref{thm:max_ineq_IS_weighted_MEP} is based on a finite-depth adaptive chaining device, in which we leverage the IS-weighted structure to carefully bound the size of the tip of the chains. In contrast, applying Theorem 8.13 of \citet{geer2000empirical} to the IS weighted sequential empirical process would lead to suboptimal dependence on $\gamma_t$.
The crucial point is to work with IS-weighted chains of the form $(g^*/g_t) \xi(f) = (g^*/g_t) \{(\xi(f) -u^{J,f}) + \sum_{j=0}^J (u^{j,f} - u^{j-1,f}) + u^{0,f}\}$, where the $u^{j,f}$ are upper brackets of the unweighted class $\Xi$, at scales $\epsilon_1 > \ldots > \epsilon_J$ (we simplify here a bit the chaining decomposition for ease of presentation compared with the proof). In adaptive chaining, the tip is bounded by the $L_1$ norm of the corresponding bracket. In our case, denoting $l^{J,f}$ the lower bracket corresponding to $u^{J,f}$, the tip is bounded by $P_{g_t} (g^*/g_t) | u^{J,f} - l^{J,f}| = P_{g^*} |u^{J,f} -l^{J,f}|$, in which we integrate out the IS ratio, thereby paying no price for it. Applying directly Theorem 8.13 of \citet{geer2000empirical}, we would be working with a bracketing of the IS weighted class $\{ (g^* / g_t) \xi(f) : f \in \F\}$. When working with generic $L_2$ brackets of the weighted class, the IS-weighting structure is lost, and we cannot do better than bounding the $L_1$ of the tip by its $L_2$ norm, which depends on $\gamma_t$.
Since in sequential settings, $\gamma_t$ generally diverges to $\infty$, an optimal dependence is paramount to obtaining tight, informative results.
Our proof technique otherwise follows the same general outlines as those of \citet[Theorem 8.13]{geer2000empirical} and \citet[Theorem A.4]{van2011minimal} (or, \citealp[Theorem 6.8]{massart2007concentration} in the iid setting). 
Like these, we too leverage
an adaptive chaining device, as pioneered by \citet{ossiander1987central}.
\end{remark}

\section{Applications to Guarantees for ISWERM}\label{sec:iswermguarantees}

We now return to ISWERM and use \cref{thm:max_ineq_IS_weighted_MEP} to obtain generic guarantees for ISWERM. We will start with so-called slow rates that give generic generalization results and then present so-called fast rates that will apply in certain settings, where a so-called variance bound is available.
Let $f_1$ be a minimizer of the population risk $R^*$ over $\F$, that is $f_1 \in \argmin_{f \in \F} R^*(f)$.

\begin{assumption}[Entropy on $\ell(\mathcal{F})$]\label{asm:entropy}
Define $\ell(\mathcal{F}): = \{\ell(f, \cdot) : f \in \mathcal{F} \}$.
There exist an envelope function $\Lambda:\cO \to \Rl$ of $\ell(\F)$, and $p>\later{\geq}0$ such that, for any $\epsilon > 0$,
$$\log\Nb(\epsilon \|\Lambda\|_{2,g^*}, \ell(\mathcal{F}), \| \cdot \|_{2,g^*} )\lesssim\epsilon^{-p}\later{\pw{\epsilon^{-p}\quad&p>0,\\\log(1/\epsilon)\quad&p=0}}.$$
\end{assumption}
The case $p<2$ corresponds to the Donsker case, and $p\geq2$ to the (possibly) non-Donsker case.
\begin{assumption}[Diameters on $\ell(\F)$]\label{asm:diameter}There exist $b_0>0$ and $\rho_0>0$ such that
$$\textstyle
\sup_{f\in\F}\magd{\ell(f,\cdot) - \ell(f_1,\cdot)}_{\infty} \leq b_0 \|\Lambda\|_{2,g^*},\qquad
\sup_{f\in\F} \|\ell(f, \cdot) - \ell(f_1,\cdot)\|_{2,g^*}\leq \rho_0 \|\Lambda\|_{2,g^*}.
$$
\later{where $\Lambda$ is the same envelope function as in \cref{asm:entropy}.}
\end{assumption}

\begin{theorem}[Slow Rates for ISWERM]\label{thm:ISWERMslow}
Suppose \cref{asm:exploration,asm:entropy,asm:diameter} hold. Then for any $\delta\in(0,1/2)$, we have that, with probability at least $1-\delta$,
\begin{align}
    &R^*(\hat f_T)-\inf_{f\in\F}R^*(f)&\\
&\lesssim \|\Lambda\|_{2,g^*} \times \pw{
\rho_0 \sqrt{\frac{\gamma^\text{avg}_T}{T}} \left\lbrace \rho_0^{-p/2} + \sqrt{\log (1/\delta)} \right\rbrace + \frac{b_0 \gamma^{\max}_T}{T} \left\lbrace \rho_0^{-p} + \log(1 / \delta) \right\rbrace\quad &p<2,\\
\left(\frac{\gamma^\text{avg}_T}{T}\right)^{\frac 1 p}+ \rho_0 \sqrt{\frac{\gamma^\text{avg}_T}{T}}  \sqrt{\log (1/\delta)}  + \frac{b_0 \gamma^{\max}_T}{T} \left\lbrace \rho_0^{-p} + \log(1 / \delta) \right\rbrace\quad \quad&p> 2.
}
\end{align}
\end{theorem}
For $p=2$ the bound is similar to the second case but with polylog terms; for brevity we omit the $p=2$ case in this paper.
\Cref{thm:ISWERMslow} suggests that the excess risk of ISWERM converges at the rate of $(\gamma^\text{avg}_T/T)^{\frac1p\wedge\frac12}$. 
For example, if $\gamma^\text{avg}_T=O(T^\beta)$ and $p<2$, we obtain $O(T^{-\frac12(1-\beta)})$. For $\beta=0$ this matches the familiar slow rate of iid settings.
However, in many cases we can obtain faster rates.
\begin{assumption}[Variance Bound]\label{asm:variance} For some $\alpha>0$, we have 
\begin{align}
    \|\ell(f,\cdot)-\ell(f_1,\cdot)\|_{2,g^*} \lesssim \|\Lambda\|_{2,g^*}
\left(\frac{R^*(f)-R^*(f_1)}{\|\Lambda\|_{2,g^*}}\right)^{\frac \alpha 2}\quad\forall f\in\F.
\end{align}
\end{assumption}
\begin{assumption}[Convexity]\label{asm:convexity} $\F$ is convex and $\ell(\cdot,O)$ is almost surely convex.
\end{assumption}
\begin{theorem}[Fast Rates for ISWERM]\label{thm:ISWERMfast}
Suppose \cref{asm:exploration,asm:entropy,asm:diameter,asm:variance,asm:convexity} hold with $p<2$. Then for any $\delta\in(0,1/2)$, we have that, with probability at least $1-\delta$,
\begin{align}
    R^*(\hat f_T)-R^*(f_1) \lesssim \|\Lambda\|_{2,g^*} \times&\left\lbrace \left( \frac{\gamma^\text{avg}_T}{T}\right)^{\frac{1}{2-\alpha + p \alpha / 2}} + \left(\frac{b_0 \gamma_T^{\max}}{T} \right)^{\frac{1}{1 + p \alpha /2}}\right.\\
    &\left.~~+ \left(\frac{\gamma^\text{avg}_T \log(1/\delta)}{T} \right)^{\frac{1}{2-\alpha}} + \frac{b_0 \gamma_T^{\max} \log(1 / \delta)}{T}\right\rbrace
\end{align}
\end{theorem}

\para{The entropy condition.}
\Cref{asm:entropy} assumes an entropy bound on the loss class $\ell(\F)$. For many loss functions, we can easily satisfy this condition by assuming an entropy condition on $\F$ itself.
\begin{assumption}[Entropy on $\F$]\label{asm:entropy2} There exists $p >\later{\geq} 0$ and an envelope function $F$ of $\mathcal{F}$ such that
$$
\log\Nb(\epsilon \|F\|_{2,g^*},\F,\|\cdot\|_{2,g^*})\lesssim\epsilon^{-p}.\later{\pw{\epsilon^{-p}\quad&p>0,\\\log(1/\epsilon)\quad&p=0.}}
$$
\end{assumption}
\begin{lemma}[Lemma 4 in \citet{bibaut2019fast}]\label{lemma:bkting_unimodal_loss}
Suppose that $\{\ell(\cdot,o):o\in\cO\}$ is a set of $\Rl\to\Rl$ unimodal functions that are equi-Lipschitz. Then \cref{asm:entropy2} implies \cref{asm:entropy}.
\end{lemma}
There are many examples of $\F$ for which bracketing entropy conditions are known.
The class of $\beta$-H\"older smooth functions (meaning having derivatives of orders up to $\mathfrak b=\sup\{i\in\mathbb Z:i<\beta\}$ and the $\mathfrak b$-order derivatives are $(\beta-\mathfrak b)$-H\"older continuous) on a compact domain in $\Rl d$ has $p=d/\beta$ \citep[Corollary 2.7.2]{van1996weak}.
The class of \emph{convex} Lipschitz functions on a compact domain in $\Rl d$ has $p=d/2$ \citep[Corollary 2.7.10]{van1996weak}.
The class of monotone functions on $\Rl$ has $p=1$ \citep[Theorem 2.7.5]{van1996weak}.
If $\F=\{f(o;\theta):\theta\in\Theta\}$, $f(o;\theta)$ is Lipschitz in $\theta$, and $\Theta\subseteq\Rl d$ is compact, then any $p>0$ holds \citep[Theorem 2.7.11]{van1996weak}.
The class of c\`adl\`ag functions $[0,1]^d \to \Rl$ with sectional variation norm (aka Hardy-Krause variation) no larger than $M > 0$ has envelope-scaled bracketing entropy $O(\epsilon^{-1} \abs{ \log(1/\epsilon)}^{2(d-1)})$ \citep{bibaut2019fast}, so \cref{asm:entropy2} holds with any $p>1$ (or, we can track the log terms). 
Since trees with bounded output range and finite depth fall in the class of c\`adl\`ag functions with bounded sectional variation norm, decision tree classes also satisfy \cref{asm:entropy2} with any $p>1$.

\section{Least squares regression using ISWERM}\label{sec:regression}

We now instantiate ISWERM for least squares regression. Consider $\Y=[-\sqrt{M},\sqrt{M}]$, for some $M > 0$, $\F\subseteq[\X\times\A\to\Y]$, and $\ell(f,o)=(y-f(x,a))^2$. If $\mathcal{F}$ is convex, strongly convex losses such as $\ell$ always yield a variance bound with respect to any population risk minimizer over $\mathcal{F}$ (see e.g. lemma 15 in \cite{bartlett2006convexity}). Let $f_1 \in \argmin_{f \in \F} R^*(f)$ be such a population risk minimizer. We present in the lemma below properties relevant for application of theorems \ref{thm:ISWERMslow} and \ref{thm:ISWERMfast}

\begin{lemma}[Properties of the square loss.]\label{lemma:properties_square_loss}
Consider the setting of the current section. The square loss $\ell$ over $\mathcal{F} \times \mathcal{O}$ satisfies the following variance bound:
\begin{align}
    \left\lVert \ell(f,\cdot) - \ell(f_1,\cdot) \right\rVert_{2,g^*} \leq 4 \sqrt{M} ( R^*(f) - R^*(f_1) )^{1/2}\ \forall f \in \mathcal{F},
\end{align}
 and  the following Lispchitz property:
\begin{align}
    | \ell(f,o) - \ell(f',o) | \leq \sqrt{M} |f(a,x) - f'(a,x)|\ \forall f, f' \in \F, o \in \cO.
\end{align}
\end{lemma}

\begin{theorem}[Least squares regression]\label{thm:iswls}
Suppose \cref{asm:exploration} holds. Suppose \cref{asm:entropy2} holds for the envelope taken to be constant equal to $\sqrt{M}$, the range of the regression functions. Then for any $\delta\in(0,1/2)$, we have that, with probability at least $1-\delta$,
\begin{align}
    {R^*(\widehat{f}_T) - R^*(f_1)}
    \lesssim 
    M\times
    \begin{cases}
    \left( \frac{\gamma^{\max}_T}{T}\right)^{\frac{1}{1+p/2}} + 
    \frac{\gamma_T^{\max} \log(1 / \delta)}{T} 
    & \text{ if } p < 2, \\
    \left(\frac{\gamma^\text{avg}_T}{T}\right)^{\frac{1}{p}} +  \frac{\gamma^{\max}_T}{T} + \sqrt{\frac{ \gamma^\text{avg}_T \log (1 /\delta)}{T}} + \frac{\gamma_T^{\max} \log(1 / \delta)}{T}  
    & \text{ if } p > 2.
    \end{cases}
\end{align}
\end{theorem}

\section{Policy Learning using ISWERM}\label{sec:policylearning}

We next instantiate ISWERM for policy learning. Consider $\Y=[-M,M]$, $\F\subseteq[\X\times\A\to\Rl]$ as in \cref{ex:policy}.
\later{We will consider two approaches to policy learning: direct policy value optimization and a reduction to cost-sensitive classification using surrogate classification losses.}
Let $\ell\later{_0}(f,o)=yf(x,a)$ and $g^*(a\mid x)=1$ so that $P_{g^*}\ell\later{_0}(f,\cdot)=\E_{p_X\times f\times p_Y}[y]=\E_{p_X}[\sum_{a\in\A}f(X,a)\mu(X,a)]$ is exactly the average outcome under a policy $f$ (or, its negative value).

\later{\subsection{Direct Policy Value Optimization}\label{sec:policydirect}}

\later{First, we consider using ISWERM with $\ell=\ell_0$ so that we are directly targeting the average outcome of interest with ISWERM.}
We first give specification-agnostic slow rates, which also close an open gap in the literature.
\begin{theorem}[ISWERM Policy Learning: slow rates]\label{thm:policyslow}
Suppose \cref{asm:exploration} holds and suppose that \cref{asm:entropy2} holds withe envelope constant equal to 1 (which is the maximal range of policies). Then for any $\delta\in(0,1/2)$, we have that, with probability at least $1-\delta$,
\begin{align}
    &{R^*(\hat f_T)-\inf_{f\in\F}R^*(f)}
\lesssim {M}\times\pw{
\sqrt{\frac{\gamma^\text{avg}_T}{T}}  \sqrt{\log (1/\delta)} + \frac{\gamma^{\max}_T}{T} \log(1 / \delta) \quad &p<2,\\
\left(\frac{\gamma^\text{avg}_T}{T}\right)^{\frac 1 p}+  \sqrt{\frac{\gamma^\text{avg}_T}{T}}  \sqrt{\log (1/\delta)}  + \frac{ \gamma^{\max}_T}{T}  \log(1 / \delta)\quad&p> 2.
}
\end{align}
\end{theorem}

\begin{remark}[Comparison to \citet{zhan2021policy}]\label{remark:zhan}
\citet{zhan2021policy} consider a deterministic policy class ($f_h(x,a)=\Ind(h(x)=a)),\,\mathcal H\subseteq[\X\to\A]$) and assume that $\int_0^1\sqrt{\log N_H(\epsilon^2,\mathcal H)}d\epsilon<\infty$ where $N_H$ is the Hamming covering. This roughly corresponds to our case $p\leq1$ if we heuristically equate bracketing numbers with Hamming covering numbers. While this equality does not hold and the two complexity measures are simply different, nonetheless, many classes that have finite Hamming entropy integral also have bracketing entropy with $p<2$. This includes, for example, policy classes given by finite-depth trees and policy classes parametrized by finite-dimensional parameter. In their setting of a finite Hamming entropy integral, \citet{zhan2021policy} show a lower bound of $\Omega((\gamma^\text{avg}_T T)^{-1/2})$ on the expected regret over all policy-learning algorithms and all logging policies satisfying \cref{asm:exploration} (see their Theorem 1), that is, $\Omega(T^{-(1-\beta)/2})$ when $\gamma^\text{avg}_T=\Omega(T^\beta)$, which matches the upper bound one would obtain under epsilon-greedy exploration with $\epsilon_t=t^{-\beta}$.
At the same time, the upper bound they provide on the expected regret incurred by their algorithm is $O(\gamma^\text{avg}_TT^{-1/2})$ (see their Corollary 2.1), that is, $O(T^{-1/2+\beta})$ when $\gamma^\text{avg}_T=O(T^\beta)$. This is a potentially significant gap in the regret rate when exploration is diminishing, $\beta>0$, as is often the case with bandit-collected data.
In contrast, in the almost-analogous case of $p<2$, our \cref{thm:policyslow} gives the rate $O((\gamma^\text{avg}_T T)^{-1/2})$ on the expected regret of policy learning with ISWERM (given by integrating the tail inequality in \cref{thm:policyslow}), that is, $O(T^{-(1-\beta)/2})$ when $\gamma^\text{avg}_T=O(T^\beta)$. This matches the rate in the lower bound of \citet{zhan2021policy} and closes this open gap.

The gap arose from the specific technical route \citet{zhan2021policy} followed (not their algorithm). For the sake of exposition, we give an explanation of the phenomenon in a non-sequential i.i.d. setting, under stationary logging policy $g_1$, and under our own notation. The same phenomenon translates to the sequential setting. Since they use a symmetrization and covering-based approach, they need to work with uniform covering-type entropies\footnote{See \citet[Chapter 2.3]{van1996weak} for an explanation of why uniform covering entropy is natural for bounding symmetrized Rademacher processes.} of the form $\sup_Q \log N(\epsilon, \mathcal{H}, L_2(Q))$ for a certain class $\mathcal{H}$, where the supremum is over all finitely supported distributions. Their approach amounts to taking $\mathcal{H}$ to be the weighted loss class $\{(g^* / g_t) \ell(\pi) \}$. While for $Q = P_{g_1}$, it holds that $\|(g^* / g_1) (\ell(\pi) - \ell(\pi')) \|_{2,P_{g_1}} \leq \gamma_1^{1/2} \|\ell(\pi) - \ell(\pi')|\|_{2,P_{g*}} \lesssim \gamma_1^{1/2} d_H(\pi, \pi')$, for a general $Q$, the best bound is $\|(g^* / g_1)  (\ell(\pi) - \ell(\pi'))\|_{2,Q} \lesssim \gamma_1 d_H(\pi, \pi')$, where $d_H$ is the Hamming distance.
By contrast, when working with bracketing entropy, one only needs to control the size of brackets in terms of $L_2(P_{g_1})$, that is the $L_2$-norm under the distribution of the data $P_{g_1}$. This allows to save a $\sqrt{\gamma_1}$ factor. 
Our results also show that a simple IS weighting algorithm suffices to obtain optimal rates, and the stabilization by $\gamma_t$ employed by \citet{zhan2021policy}, which is inspired by the stabilization employed by \citet{hadad2019confidence,luedtke_vdL2016} for inference purposes, may not be necessary for policy learning purposes. The doubly-robust-style centering may still be beneficial in practice for reducing variance but it does not affect the rate.
\end{remark}

\begin{remark}[Comparison to \cite{foster2018contextual}]
\citet{foster2018contextual}, albeit in a slightly different setting, derive a maximal inequality under sequential covering entropy that also exhibits the correct dependence on the exploration rate as ours. This shows in particular that the suboptimal dependence on the exploration rate of \citet{zhan2021policy} is not a necessary consequence of using sequential covering entropy. Analogously to us, \citet{foster2018contextual} exploit the specific IS-weigthed structure of the loss process, and work with covers of the unweighted policy class directly. Using an $L_\infty$ sequential cover of the unweighted class and using Holder's inequality, they are able to factor out the $L_1$ norm of the IS ratios. This allows them to circumvent the type of sequential cover of the weighted class that  \citet{zhan2021policy} need, and yields optimal $\gamma$ scaling. One caveat of this approach is that the entropy integral in the corresponding bound is expressed in terms of $L_\infty$ sequential covering entropy, which makes it hard to obtain fast rates via localization. Indeed, while variance bounds that allow for localization in $L_2$ norm are common, it is in general much harder to obtain localization in $L_\infty$ norm.
\end{remark}

In well-specified cases, much faster rates of regret convergence are possible. 
We focus on finitely-many actions, $\abs{\A}<\infty$.
Define $\mu^*(X)=\min_{a\in\A}\mu(X,a)$ and fix $a^*(X)$ with $\mu(X,a^*(X))=\mu^*(X)$.
\begin{assumption}[Margin]\label{asm:margin} 
For a constant $\nu\in[0,\infty]$, we have for all $u\geq0$,
$$\textstyle
\Pr_{p_X}\prns{
\min_{a\in\A\backslash\{a^*(X)\}}\mu(X,a)-\mu^*(X)
\leq Mu}^{1/\nu}\lesssim u,
$$
where we define $0^{1/\infty}=0$ and $x^{1/\infty}=1$ for $x\in(0,1]$.
\end{assumption}
This type of margin condition was originally considered in the case of binary classification \citep{mammen1999smooth,tsybakov2004optimal}.
The condition we use is more similar to that used in multi-arm contextual bandits \citep{hu2020smooth,perchet2013multi,hu2020fast}.
The condition controls the density of the arm gap near zero.
It generally holds with $\nu=1$ for sufficiently well-behaved $\mu$ and continuous $X$ and with $\nu=\infty$ for discrete $X$ \citep[see, \eg,][Lemmas 4 and 5]{hu2021fast}.

\begin{lemma}\label{lemma:vb_from_margin_cond} Suppose \cref{asm:margin} holds and $\min_{f\in\F}R^*(f)=\E_{p_X}\mu^*(X)$. Then \cref{asm:variance} holds for $\alpha=\nu/(\nu+1)$ and $\Lambda: o \mapsto M$.
\end{lemma}

\begin{theorem}[ISWERM Policy Learning: fast rates]\label{thm:policyfast}
Suppose \cref{asm:exploration,asm:entropy2,asm:margin} hold with $p<2$ and $\min_{f\in\F}R^*(f)=\E_{p_X}\mu^*(X)$. Then for any $\delta\in(0,1/2)$, with probability at least $1-\delta$,
\begin{align}
    {R^*(\hat f_T)-\E_{p_X}\mu^*(X)} \lesssim &~ {M}\left(\left( \frac{\gamma^\text{avg}_T}{T}\right)^{\frac{1+\nu}{2 + \nu(1 + p /2)}} + \left(\frac{ \gamma_T^{\max}}{T} \right)^{\frac{1 + \nu}{1 + \nu(1 + p / 2)}}
    \right.
    \\&~~~~~~~~~
    \left.
    + \left(\frac{\gamma^\text{avg}_T \log(1/\delta)}{T} \right)^{\frac{1+\nu}{2+\nu}} + \frac{ \gamma_T^{\max} \log(1 / \delta)}{T}\right).
\end{align}
\end{theorem}

\begin{remark}[Classification using ISWERM]
The above results can easily be rephrased for the classification analogue to the regression problem in \cref{sec:regression}, where $\Y=\{\pm1\}$ and we want a classifier based on features $x,a$ to minimize misclassification error. Because the policy learning problem is both of greater interest and greater generality, we focus our presentation on policy learning.
\end{remark}

\later{\subsection{Cost-Sensitive Classification}

An issue with the algorithm explored in \cref{sec:policydirect} is that the optimization problem to solve ISWERM can be difficult. Although it is convex (in fact linear) in $f$, we usually parametrize the policy $f$ as the argmax or softmax of a function assigning to each action a score, given $x$. Then, ISWERM for direct policy value optimization over such a score function becomes nonconvex and difficult, akin to classification using the binary loss. In this section, we explore an alternative approach using a reduction to cost-sensitive classification and using surrogate loss functions, as originally suggested by \citet{zhao2012estimating,jiang2019entropy,beygelzimer2009offset} in the iid setting.

\section{Adaptive Leave-One-Out Cross-Validation}\label{sec:loo}}

\section{Empirical Study}
\label{sec:empirics}

Next, we empirically investigate various risk minimization approaches using data collected by a contextual bandit algorithm, including both ISWERM and unweighted ERM among others. We take 51 different mutli-calss classification dataset from OpenML-CC18 \citep{bischl2017openml} and transform each into a multi-arm contextual bandit problem \citep[following]{dudik2014doubly,dimakopoulou2017estimation,su2019cab}. We then run an epsilon greedy algorithm for $T = 100000$, where we explore uniformly with probability $\epsilon_t=t^{-1/3}$ and otherwise pull the arm that maximizes an estimate of $\mu(x,a)$ based on data so far.
Details are given in \cref{sec:dataconstruction}.

We then consider using this data to regress $Y_t$ on $X_t,A_t$ using different methods where each observation is weighted by $w_t$ using different schemes:
(1) Unweighted ERM: $w_t=1$;
(2) ISWERM: $w_t=g_t^{-1}(A_t\mid X_t)$;
(3) ISFloorWERM: $w_t=\gamma^{-1}_t$, where--inspired by \citep{zhan2021policy}--we weight by the inverse (nonrandom) floor $\gamma_t = \epsilon_t / |\mathcal{A}|$ of the propensity scores;
(4) SqrtISWERM: $w_t=g_t^{-1/2}(A_t\mid X_t)$, which applies the stabilization of \citep{luedtke_vdL2016,hadad2019confidence} to ISWERM;
(5) SqrtISFloorWERM: $w_t=\gamma^{-1/2}_t$;
(6) MRDRWERM: $w(t) = \frac{1-g_t(A_t \mid X_t)}{g_t^2(A_t \mid X_t)}$, which are the weights used by \citet{farajtabar2018more};
(7) MRDRFloorWERM: $w(t) = \frac{1-\gamma_t}{\gamma_t^2}$, which is like MRDRWERM but uses the 
propensity score 
floors $\gamma_t$.
With these sample weights, we run either Ridge regression, LASSO, or CART using \verb|sklearn|'s \verb|RidgeCV(cv=4)|, \verb|LassoCV(cv=4)|, or \verb|DecisionTreeRegressor|, each with default parameters. For Ridge and LASSO we pass as features the intercept-augmented contexts $\{(1, X_t)\}_{t=1}^T$ concatenated by the product of the one-hot encoding of arms $\{A_t\}_{t=1}^T$ with the intercept-augmented contexts $\{(1, X_t)\}_{t=1}^T$. For CART, we use the concatenation of the contexts $\{X_t\}_{t=1}^T$ with the one-hot encoding of arms $\{A_t\}_{t=1}^T$.
To evaluate, we play our bandit anew for $T^\text{test}=1000$ rounds using a uniform exploration policy, $g^*(a\mid x)=1/K$, and record the mean-squared error (MSE) of the regression fits on this data.
We repeat the whole process 64 times and report estimated average MSE and standard error in \cref{fig:mse-linearbandit}.

\para{Results.}
\Cref{fig:mse-linearbandit-lasso,fig:mse-linearbandit-ridge} show that ISWERM clearly outperforms unweighted ERM and all other weighted-ERM schemes for linear regression, with ISWERM's advantage being even more pronounced for LASSO. Intuitively, since a linear model is misspecified, this can be attributed to ISWERM's ability to provide agnostic best-in-class risk guarantees. In contrast, for a better specified model such as CART, all ERM methods perform similarly, as seen in \cref{fig:mse-linearbandit-tree}. We highlight that our focus is not necessarily methodological improvements, and the aim of our experiments is to explore the implications of our theory, not provide state-of-the-art results.
We provide additional empirical results in \cref{apdx:empirics}, the conclusions from which are qualitatively the same.


\begin{figure}
    \centering
\subfloat[LASSO outcome model with cross-validated regularization parameter.\label{fig:mse-linearbandit-lasso}]{\includegraphics[width=1\linewidth]{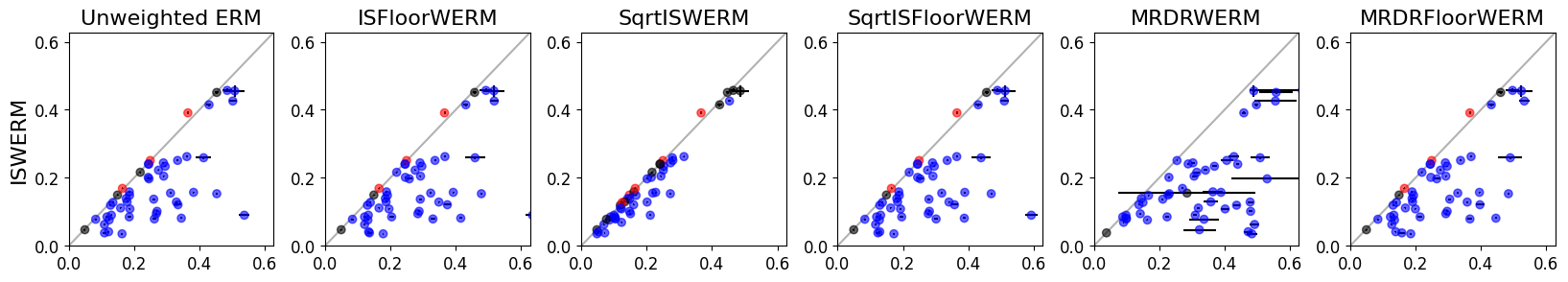}} \\
\subfloat[Ridge outcome model with cross-validated regularization parameter.\label{fig:mse-linearbandit-ridge}]{\includegraphics[width=1\linewidth]{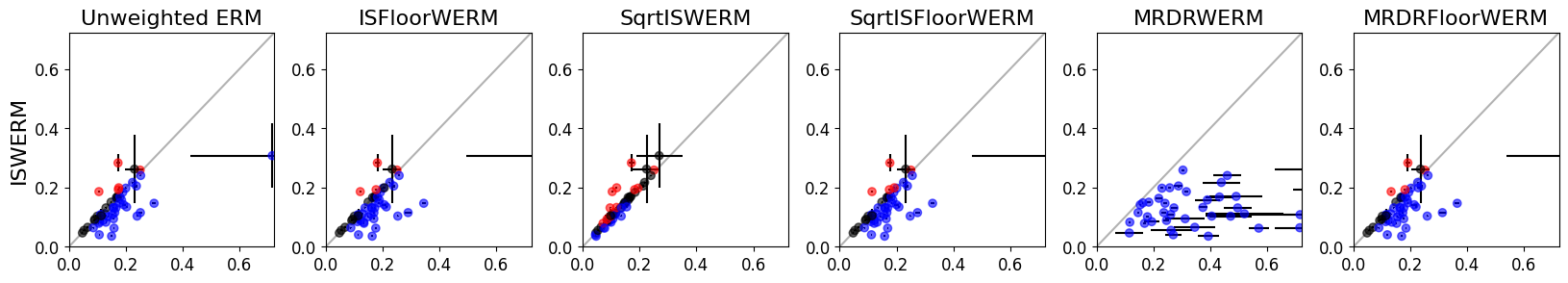}} \\
\subfloat[CART outcome model with unrestricted tree depth.\label{fig:mse-linearbandit-tree}]{
\includegraphics[width=1\linewidth]{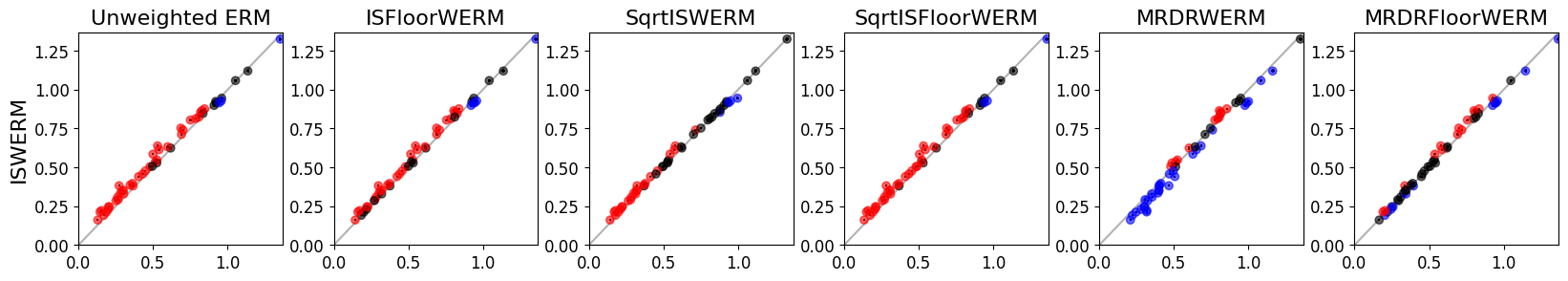}}
  \caption{Comparison of weighted regression run on contextual-bandit-collected data. Each dot is one of 51 OpenML-CC18 datasets. Lines denote $\pm1$ standard error. Dots are blue when ISWERM is clearly better, red when clearly worse, and black when indistinguishable within one standard error.}
    \label{fig:mse-linearbandit}
\end{figure}


\section{Conclusions and Future Work}

We provided first-of-their-kind guarantees for risk minimization from adaptively collected data using ISWERM. Most crucially, our guarantees provided good dependence on the size of IS weights leading to correct convergence rates when exploration diminishes with time, as happens when we collect data using a contextual bandit algorithm. This was made possible by a new maximal inequality specifically for IS weighted sequential empirical processes. There are several important avenues for future work. We focused on a fixed hypothesis class. One important next question is how to do effective model selection in adaptive settings. We also focused on IS weighted regression and policy learning, but recent work in the iid setting highlights the benefits of using doubly-robust-style centering \citep{foster2019orthogonal,kennedy2020optimal,athey2021policy}. These benefits are most important to avoid rate deterioration when IS weights are estimated, while our IS weights are known, but there are still benefits in reducing the loss variance in the leading constant. Therefore, exploring such methods in adaptive settings is another important next question.

\section{Societal Impact}
Our work provides guarantees for learning from adaptively collected data. While the methods (IS weighting) are standard, our novel guarantees lend credibility to the use of adaptive experiments.
Adaptive experiments hold great promise for better, more efficient, and even more ethical experiments.
At the same time, adaptive experiments, especially when all arms are always being explored ($\gamma_t<\infty$) even if at vanishing rates ($\gamma_t=\omega(1)$), must still be subject to the same ethical guidelines as classic randomized experiments regarding favorable risk-benefit ratio of any arm, informed consent, and other protections of participants.
There are also several potential dangers to be aware of in supervised and policy learning generally, such as the training data possibly being unrepresentative of the population to which predictions and policies will be applied leading to potential disparities as well as the focus on \emph{average} welfare compared to prediction error or policy value on each individual or group.
These remain concerns in the adaptive setting, and while ways to tackle these challenges in non-adaptive settings might be applicable in adaptive ones, a rigorous study of applicability requires future work.

\bibliographystyle{plainnat}
\bibliography{AdaptiveERM}

\newpage
\appendix

\begin{center}\large Supplementary Material for:\\{\Large\bf
Risk Minimization from Adaptively Collected Data:\\Guarantees for Supervised and Policy Learning}
\end{center}

\section{Proof of the maximal inequality for IS weighted sequential empirical processes}

\subsection{Preliminary lemmas}

For any sequence $\widetilde{g}_1,\ldots,\widetilde{g}_T$ of conditional densities and any finite sequence $\zeta_{1:T} := (\zeta_t)_{t=1}^T$ of $\cO \to \Rl$ functions, let
\begin{align}
    \rho_{T,\widetilde{g}_{1:T}}(\zeta_{1:T}) := \left( \frac{1}{T} \sum_{t=1}^T \|\zeta_t\|_{2,\widetilde{g}_t}^2 \right)^{1/2}.
\end{align}
For any conditional density $\widetilde{g}:(a,x) \in \mathcal{A} \times \mathcal{X} \mapsto \widetilde{g}(a \mid x)$, let 
\begin{align}
    \rho_{T, \widetilde{g}}(\zeta_{1:T}) := \rho_{T, \widetilde{g}_{1:T}}(\zeta_{1:T}),
\end{align}
where we set $\widetilde{g}_t := \widetilde{g}$ for every $t \in [T]$.

\begin{lemma}\label{lemma:rho_pseudonorm}
Any $\rho_{T,\widetilde{g}_{1:T}}$ as defined above is a pseudonorm over the vector space $(\cO \to \Rl)^T$.
\end{lemma}

\begin{proof}[Proof of \cref{lemma:rho_pseudonorm}]
It is immediate that for any real number $\lambda$, and finite sequence $\zeta_{1:T}$ of $\cO \to \Rl$ functions $\rho_{T, \widetilde{g}_{1:T}}(\lambda \zeta_{1:T}) = | \lambda | \rho_{T, \widetilde{g}_{1:T}}(\zeta_{1:T})$.

We now check that $\rho_{T,\widetilde{g}_{1:T}}$ satisfies the triangle inequality. Let $\zeta^{(1)}_{1:T}$ and $\zeta^{(2)}_{1:T}$ be two sequences of $\cO \to \Rl$ functions. We have that
\begin{align}
    \rho_{T, \widetilde{g}_{1:T}}(\zeta^{(1)}_{1:T} + \zeta^{(2)}_{1:T}) =& \left( \frac{1}{T} \sum_{t=1}^T \left\lVert \zeta^{(1)}_t + \zeta^{(2)}_t \right\rVert_{2, \widetilde{g}_t}^2 \right)^{1/2} \\
    \leq& \left( \frac{1}{T} \sum_{t=1}^T \left(\left\lVert \zeta^{(1)}_t \right\rVert_{2,\widetilde{g}_t} + \left\lVert \zeta^{(2)}_t \right\rVert_{2,\widetilde{g}_t} \right)^2 \right)^{1/2} \\
    \leq & \left( \frac{1}{T} \sum_{t=1}^T \left\lVert \zeta^{(1)}_t \right\rVert_{2,\widetilde{g}_t}^2 \right)^{1/2}  +\left( \frac{1}{T} \sum_{t=1}^T \left\lVert \zeta^{(2)}_t \right\rVert_{2,\widetilde{g}_t} \right)^{1/2} \\
    =& \rho_{T, \widetilde{g}_{1:T}}(\zeta^{(1)}_{1:T}) + \rho_{T, \widetilde{g}_{1:T}}(\zeta^{(2)}_{1:T}),
\end{align}
where the second line above follows from the triangle inequality applied to the pseudonorms $\|\cdot\|_{2,\widetilde{g}_t}$, $t=1,\ldots,T$, and where the third line follows from the triangle inequality applied to the Euclidean norm $x \in \Rl^T \mapsto (\sum_{t=1}^T x_t^2)^{1/2}$.
\end{proof}

\begin{lemma}\label{lemma:bound_rho_norm_IS_weigthed_sequence}
Consider $g^*$ and $g_1,\ldots,g_T$ as defined in the main text. Suppose that assumption \ref{asm:exploration} holds. Then, for any finite sequence of functions $(\zeta_t)_{t=1}^T \in (\cO \to \Rl)^T$,
\begin{align}
    \rho_{T,g_{1:T}}\left( \frac{g^*}{g_{1:T}} \zeta_{1:T} \right) \leq \sqrt{\gamma_T^{\max}} \rho_{T,g^*}(\zeta_{1:T}).
\end{align}
If all elements of the sequence $\zeta_t$ are the same, that is, if there exists $\zeta:\cO \to \Rl$ such that $\zeta_t = \zeta$ for every $t \in [T]$, then
\begin{align}
    \rho_{T,g_{1:T}}\left( \frac{g^*}{g_{1:T}} \zeta_{1:T} \right) \leq \sqrt{\gamma^\text{avg}_T} \|\zeta\|_{2,g^*}.
\end{align}
\end{lemma}

\begin{proof}[Proof of \cref{lemma:bound_rho_norm_IS_weigthed_sequence}]
We have that 
\begin{align}
    \rho_{T,g_{1:T}} \left( \frac{g^*}{g_{1:T}}  \zeta_{1:T} \right) =& \left( \frac{1}{T} \sum_{t=1}^T P_{g_t} \left( \frac{g^*}{g_t} \zeta_t \right)^2 \right)^{1/2} \\
    =& \left( \frac{1}{T} \sum_{t=1}^T P_{g^*} \left( \frac{g^*}{g_t} \zeta_t^2 \right) \right)^{1/2} \\
    \leq & \left( \frac{1}{T} \sum_{t=1}^T \gamma_t P_{g^*} \zeta_t^2  \right)^{1/2},
\end{align}
where the inequality follows from \cref{asm:exploration}.
If there exists $\zeta : \cO \to \Rl$ such that $\zeta_t = \zeta$ for every $t=1,\ldots,T$, then, 
\begin{align}
    \rho_{T,g_{1:T}} \leq & \left( \frac{1}{T} \sum_{t=1}^T \gamma_t P_{g^*} \zeta_t^2  \right)^{1/2} \\
    =& \sqrt{\gamma_T^\text{avg}} \left\lVert \zeta \right\rVert_{2,g^*}.
\end{align}
Otherwise, we have
\begin{align}
\rho_{T,g_{1:T}} \leq & \left( \frac{1}{T} \sum_{t=1}^T \gamma_t P_{g^*} \zeta_t^2  \right)^{1/2} \\
\leq & \left( \max_{t \in [T]} \gamma_t \frac{1}{T} \sum_{t=1}^T P_{g^*} \zeta_t^2  \right)^{1/2} \\
=& \sqrt{\gamma_T^{\max}} \rho_{T,\widetilde{g}_{1:T}}(\zeta_{1:T}).
\end{align}
\end{proof}

The following lemma is a restatement under our notation of Corollary A.8 in \citet{van2011minimal}.
\begin{lemma}\label{lemma:max_ineq_finite_class}
Let $\zeta^1_{1:T},\ldots,\zeta^N_{1:T}$ be $N$ $\bar{O}_{1:T}$-predictable sequences of $\cO \to \Rl$ functions, and let $A$ be an $\bO_T$-measurable event. Then, for any $r > 0$ and any $b > 0$ such that $\max_{i \in [N], t \in [T]} \|\zeta^i_t\|_\infty \leq b$, it holds that 
\begin{align}
    & E \left[ \max_{i \in [N]} \frac{1}{T} \sum_{t=1}^N (\delta_{O_t} - P_{g_t}) \zeta^i_{t} \Ind (\rho_{T, g_{1:T}} (\zeta_{1:T}^i) \leq r) \mid A \right] \\
    \lesssim & r \sqrt{\frac{\log (1 + N / P[A])}{t}} + \frac{B}{t} \log (1 + N/ P[A]).
\end{align}
\end{lemma}

\subsection{Proof of \cref{thm:max_ineq_IS_weighted_MEP}}

\begin{proof}[Proof of \cref{thm:max_ineq_IS_weighted_MEP}]

We treat together both the general case where, for each $f$, $\xi_{1:T}(f)$ is an $\bar{O}_{1:T}$-predictable sequence, and the case where, for every $f$, there exists a deterministic $\xi(f):\cO \to \Rl$ such that $\xi_t(f) = \xi(f)$ for every $t \in [T]$. We refer to the former as \textit{case 1} and to the latter as \textit{case 2} in the rest of the proof. In case 1, we let $\widetilde{\rho}_T := \rho_T^{\max}$, and in case 2, we let $\widetilde{\rho}_T := \bar{\rho}_T$.

\paragraph{From a conditional expectation bound to a high probability bound.} 

Let $x > 0$. We introduce the following event:
\begin{align}
    A := \left\lbrace \sup_{f \in \mathcal{F}} M_T(f) \geq \psi(x) \right\rbrace,
\end{align}
where 
\begin{align}
    \psi(x) := & C \left\lbrace r^- + \sqrt{\frac{\widetilde{\gamma}_T}{T}} \int_{r^-}^r \sqrt{\log (1 + \mathcal{N}_{[\,]}(\epsilon, \Xi_T, \rho_{T,g^*}))} d\epsilon  \right.\\
    & \qquad + \frac{B \gamma^{\max}_T}{ T} \log (1 + \mathcal{N}_{[\,]}(r, \Xi_T, \rho_{T,g^*})) \\
    & \qquad \left. r \sqrt{\frac{x}{T}} + \frac{\gamma^{\max}_T x}{ T} \right\rbrace,
\end{align}
where $C$ is a universal constant to be discussed further down.
Suppose we can show that 
\begin{align}
    E \left[ \sup_{f \in \mathcal{F}} M_T(f) \mid A \right] \leq \psi\left( \log \left( 1 + \frac{1}{P[A]} \right) \right).
\end{align}
Then, we will have that $\psi(x) \leq \psi( \log (2 / P[A]))$, that is $P[A] \leq 2 e^{-x}$, which is the wished claim.

\paragraph{Setting up the chaining decomposition.} Let $\epsilon_0 := r$, and, for every $j \geq 0$, let $\epsilon_j := \epsilon_0 2^{-j}$. For any $j \geq 0$, let
\begin{align}
    \mathcal{B}^j := \left\lbrace (\lambda_s^{j,k}, \upsilon^{j,k}_t)_{t=1}^T : k \in [N_j] \right\rbrace
\end{align}
 be a minimal $(\epsilon_j, \rho_{T,g^*})$-sequential bracketing of $\Xi_T$. For any $f \in \mathcal{F}$, let $k(j,f) \in [N_j]$ be such that 
 \begin{align}
     \lambda_s^{j,k(j,f)} \leq \xi_t(f) \leq \upsilon^{j,k(j,f)} \text{ for every } t \in [T],
 \end{align}
 and let $\Delta_t^{j,f} := \upsilon_s^{j,k(j,f)} - \lambda_s^{j,k(j,f)}$ and $u^{j,f} := \upsilon_s^{j,k(j,f)}$. For any $j \geq 0$, let $\bar{N}_j := \prod_{i=0}^j N_i$. For any $j \geq 0$, and $t \in [T]$ let
 \begin{align}
     a_{j,t} := \epsilon_j \sqrt{\frac{T}{ \log (1 + \bar{N}_j / P[A])}} \frac{\sqrt{\widetilde{\gamma}_T}}{\gamma_t}.
 \end{align}
 Let $J \geq 0$ such that $\epsilon_{J+1} < r^- \leq \epsilon_J$. The integer $J$ will be the maximal depth of the chains in our chaining decomposition. For any $t \in [T]$, $f \in \mathcal{F}$, let 
 \begin{align}
     \tau_t(f) := \inf \left\lbrace j \geq 0 : \Delta_t^{j,f} > a_{j,t} \right\rbrace \wedge J,
 \end{align}
be the depth at which we truncate the chains, adaptively depending on the value of $\Delta_t^{j,f}$, so that $\Delta_t^{j,f} \Ind(\tau_t(f) > j)$ is no larger than $a_{j,t}$ in supremum norm at any depth $j$.

For any $f \in \mathcal{F}$ and any $t \in [T]$, the following chaining decomposition holds:
\begin{align}
    \xi_t(f) =& \underbrace{\sum_{j=0}^J (\xi_t(f) - u^{j,f} \wedge u^{j-1,f}) \Ind (\tau_t(f) = j )}_{\text{tip of the chain}} \\
    &+ \underbrace{\sum_{j=1}^J \left\lbrace (u^{j,f} \wedge u^{j-1,f} - u^{j-1,f})\Ind(\tau_t(f) = j) + (u^{j,f} - u^{j-1,f}) \Ind(\tau_t(f) > j) \right\rbrace }_{\text{links of the chain}} \\
    &+ \underbrace{u^{0,f}_t}_{\text{root of then chain}}.
\end{align}

\paragraph{Control of the tips.} 
\begin{itemize}
    \item \textbf{Case $j=J$.} We have that
    \begin{align}
        &\frac{1}{T} \sum_{t=1}^T (\delta_{O_t} - P_{g_t}) \frac{g^*}{g_t}(\xi_t(f) - u_t^{J,f} \wedge u_t^{J-1,f}) \Ind(\tau_t(f) = J) \\
        \leq & \frac{1}{T} \sum_{t=1}^T P_{g_t} \frac{g^*}{g_t} \Delta_t^{J,f} \\
        =& \frac{1}{T} \sum_{t=1}^T \|\Delta_t^{J,f}\|_{1,g^*} \\
        \leq & \left( \frac{1}{T} \sum_{t=1}^T \|\Delta_t^{J,f}\|_{2,g^*}^2 \right)^{1/2} \\
        \leq & \epsilon_J.
    \end{align}
    Therefore 
    \begin{align}
        E \left[ \sup_{f \in \mathcal{F}} \frac{1}{T} \sum_{t=1}^T (\delta_{O_t - P_{g_t}} \frac{g^*}{g_t})(\xi_t(f) - u_t^{J,f} \wedge u_t^{J-1,f}) \Ind(\tau_t(f) = J) \mid A \right] \leq \epsilon_J.
    \end{align}
    \item \textbf{Case $j < J$.}
    \begin{align}
         &\frac{1}{T} \sum_{t=1}^T (\delta_{O_t} - P_{g_t}) \frac{g^*}{g_t}(\xi_t(f) - u_t^{j,f} \wedge u_t^{j-1,f}) \Ind(\tau_t(f) = j) \\
          \leq & \frac{1}{T} \sum_{t=1}^T P_{g_t} \frac{g^*}{g_t} \Delta_t^{j,f} \Ind(\tau_t(f) = j) \\
          \leq & \frac{1}{T} \sum_{t=1}^T P_{g^*} \frac{(\Delta_t^{j, f})^2}{a_{j,t}} \\
          \leq & \epsilon_j^2 \frac{1}{T} \sum_{t=1}^T \frac{1}{a_{j,t}} \\
          =& \epsilon_j \sqrt{\frac{\log (1 + \bar{N}_j / P[A])}{T}} \frac{1}{\sqrt{\widetilde{\gamma}_T}} \frac{1}{T} \sum_{t=1} \gamma_t \\\
          \leq & \epsilon_j \sqrt{ \frac{\widetilde{\gamma}_T \log (1 + \bar{N}_j / P[A])}{T}}.
    \end{align}
    (The last inequality is an equality in \textit{case 2}).
\end{itemize}

\paragraph{Control of the links.} We start with bounding the $\rho_{T, g_{1:T}}$ pseudo-norm of the IS weighted links.
We have that 
\begin{align}
    & \rho_{T,g_{1:T}}\left( \left(\frac{g^*}{g_t}(u^{j,f}_t\wedge u^{j-1,f}_t - u^{j-1,f}_t \right)_{t=1}^T\right) \\
    \leq & \rho_{T,g_{1:T}}\left( \left(\frac{g^*}{g_t}(u^{j,f}_t - u^{j-1,f}_t \right)_{t=1}^T\right) \\
    \leq & \sqrt{\widetilde{\gamma}_T} \rho_{T,g^*}\left( u^{j,f}_{1:T} - u^{j-1,f}_{1:T} \right) \\
    \leq & \sqrt{\widetilde{\gamma}_T} \left\lbrace \rho_{T,g^*}\left( u^{j,f}_{1:T} - \xi_{1:T}(f) \right) +   \rho_{T,g^*}\left(\xi_{1:T}(f) - u^{j-1,f}_{1:T} \right) \right\rbrace\\
    \lesssim  &\sqrt{\widetilde{\gamma}_T} \epsilon_j,
\end{align}
where we have used lemma \ref{lemma:bound_rho_norm_IS_weigthed_sequence} is the third line and where the fourth line above follows from the triangle inequality.

We now bound the supremum norm of the links. For every $t \in [T]$,
\begin{align}
    &(u_t^{j,f} \wedge u_t^{j-1, f} - u_t^{j,f}) \Ind(\tau_t(f) = j) \\
    =& (u_t^{j,f} \wedge u_t^{j-1, f} - \xi_t(f)) \Ind(\tau_t(f) = j) \\
    &-  ( u_t^{j-1, f} - \xi(f)) \Ind(\tau_t(f) = j).
\end{align}
Using the definition of $\tau_t(f)$, we obtain
\begin{align}
    0 \leq  (u_t^{j,f} \wedge u_t^{j-1,f} - \xi_t(f) ) \Ind(\tau_t(f) = j) 
    \leq  (u_t^{j-1, f} - \xi_t(f)) \Ind(\tau_t(f) = j) 
    \leq  a_{j-1,t} 
    \lesssim  a_{j,t},
\end{align}
and 
\begin{align}
    0 \leq (u_t^{j-1,f} - \xi_t(f)) \Ind(\tau_t(f)=j) \leq a_{j-1,t} \lesssim a_{j-1,t}.
\end{align}
Therefore,
\begin{align}
    \max_{t\in [T]} \left\lVert \frac{g^*}{g_t} \left(u_t^{j,f} \wedge u_t^{j-1,f} - u_t^{j-1,f}\right)\Ind(\tau_t(f) = j) \right\rVert_\infty \lesssim \gamma_t a_{j,t}=  b_j
\end{align}
where 
\begin{align}
    b_j := \epsilon_j \sqrt{\frac{T \widetilde{\gamma}_T}{\log (1 + \bar{N}_j / P[A])}}
\end{align}
Similarly, we have
\begin{align}
    0 \leq (u_t^{j,f} - \xi_t(f)) \Ind(\tau_t(f) > j) \leq a_{j,t} 
    \qquad \text{and} \qquad 0 \leq (u_t^{j-1,f} - \xi_t(f)) \Ind(\tau_t(f) > j) \leq a_{j-1,t},
\end{align}
and therefore, for every $t \in [T]$
\begin{align}
   \left\lVert \frac{g^*}{g_t} \left( u_t^{j-1,f} - u_t^{j-1,f}\right)\Ind(\tau_t(f) > j) \right\rVert_\infty \lesssim  \gamma_t a_{j,t} = b_j
\end{align}
Denote
\begin{align}
    v_t^{j,f} := \frac{g^*}{g_t} \left\lbrace (u_t^{j,f} \wedge u_t^{j-1,f} - u_t^{j,f})\Ind(\tau_t(f) = j)  + ( u_t^{j,f} - u_t^{j-1,f})\Ind(\tau_t(f) > j)\right\rbrace.
\end{align}
Observe that as $f$ varies over $\mathcal{F}$, $v_{1:T}^{j,f}$ varies over a collection of at most $N_j \times N_{j-1} \leq \bar{N}_j$ elements. Therefore, lemma \ref{lemma:max_ineq_finite_class} yields
\begin{align}
    &E \left[\sup_{f \in \mathcal{F}} \frac{1}{T} \sum_{t=1}^T (\delta_{O_t} - P_{g_t}) \frac{g^*}{g_t} v_t^{j,f} \right] \\
    \lesssim &\epsilon_j \sqrt{\frac{\widetilde{\gamma}_T \log (1 + \bar{N}_j / P[A])}{ T}} + \frac{b_j}{T} \log (1 + \bar{N}_j / P[A]) \\
    \lesssim &\epsilon_j \sqrt{\frac{\widetilde{\gamma}_T \log (1 + \bar{N}_j / P[A])}{ T}}.
\end{align}

\paragraph{Control of the root.}
For any $f$ such that $\rho_{T,g^*}((\xi_t(f))_{t=1}^T) \leq r$, we have that
\begin{align}
    &\rho_{T,g_{1:T}} (((g^* / g_t) u^{0,f}_t)_{t=1}^T) \\
    \leq & \sqrt{\widetilde{\gamma}_T} \rho_{T,g^*} (u^{0,f}_{1:T}) \\
    \leq &\sqrt{\widetilde{\gamma}_T} (\rho_{T,g^*} (u^{0,f}_{1:T} - \xi_{1:T}(f)) + \rho_{T,g^*}(\xi_{1:T}(f)).
\end{align}
Without loss of generality, we can assume that $\max_{t \in [T]} \| u^{0,f}_t\|_\infty \leq B$, since thresholding to $B$ preserves the bracketing property. Therefore, $\max_{t \in [T]} \| (g^* / g_t) u^{0,f}_t\|_\infty \leq \gamma_T^{\max} B\epsilon$.

Then, from lemma \ref{lemma:max_ineq_finite_class}, 
\begin{align}
    &E \left[ \sup \left\lbrace \frac{1}{T} \sum_{t=1}^T (\delta_{O_t} - P_{g_t}) \xi_t(f) : f \in \mathcal{F}, \rho_{T, g^*}((\xi_t(f))_{t=1}^T) \leq r \right\rbrace \right] \\
    \leq & \sqrt{\frac{\widetilde{\gamma}_T}{T}} \sqrt{\log \left((1 + \frac{\bar{N}_0}{P[A]}\right)} + \frac{B \gamma_T^{\max}}{T} \log \left(1 + \frac{\bar{N}_0}{P[A]}\right)
\end{align}

\paragraph{Adding up the bounds.}

We obtain
\begin{align}
    E \left[ \sup_{f \in \mathcal{F}} M_T(f) \mid A\right] \lesssim & \underbrace{\sqrt{\frac{\widetilde{\gamma}_T}{T}} \sqrt{\log \left((1 + \frac{\bar{N}_0}{P[A]}\right)} + \frac{B}{\delta T} \log \left(1 + \frac{\bar{N}_0}{P[A]}\right)}_{\text{root contribution}} \\
    &+ \underbrace{\sqrt{\frac{\widetilde{\gamma}_T}{T}}\sum_{j=1}^J \epsilon_j \log \left( 1 + \frac{\bar{N}_j}{P[A]} \right)}_{\text{links contribution}} \\
    &+ \underbrace{\sqrt{\frac{\widetilde{\gamma}_T}{T}}\sum_{j=0}^{J-1} \epsilon_j \log \left( 1 + \frac{\bar{N}_j}{P[A]} \right) + \epsilon_J}_{\text{tip contribution}} \\
    \lesssim & \epsilon_J + \sqrt{\frac{\widetilde{\gamma}_T}{T}} \sum_{j=0}^J \epsilon_j \log \left( 1 + \frac{\bar{N}_j}{P[A]} \right) + \frac{B \gamma_T^{\max}}{ T} \log \left(1 + \frac{\bar{N}_0}{P[A]}\right). 
\end{align}

We use the classical technique from finite adaptive chaining proofs to bound the sum in the second term with an integral \citep[see e.g.][]{van2011minimal, bibaut2020}. We obtain
\begin{align}
    \sum_{j=0}^J \epsilon_j \log \left( 1 + \frac{\bar{N}_j}{P[A]} \right) \lesssim \int_{r^-}^r \sqrt{\log (1 + N_{[\,]}(\epsilon, \Xi_T, \rho_{T,g^*}))} d\epsilon + \log \left(1 + \frac{1}{P[A]}\right).
\end{align}
Therefore,
\begin{align}
    E \left[ \sup_{f \in \mathcal{F}} M_T(f) \mid A\right] \lesssim & r^- + \sqrt{\frac{\widetilde{\gamma}_T}{T}} \int_{r^-}^r \sqrt{\log (1 + N_{[\,]}(\epsilon, \Xi_T, \rho_{T,g^*}))} d\epsilon\\
    &+ \frac{B \gamma_T^{\max}}{ T} \log (1 + N_{[\,]}(r, \Xi_T, \rho_{T,g^*}))\\
    &+ \sqrt{\frac{\widetilde{\gamma}_T}{T}} \sqrt{\log \left(1 + \frac{1}{P[A]} \right)} + \frac{B \gamma_T^{\max}}{ T} \log \left(1 + \frac{1}{P[A]} \right).
\end{align}
Therefore, for an appropriate choice of the universal constant $C$ in the definition of $\psi$, we have that 
\begin{align}
    E \left[ \sup_{f \in \mathcal{F}} M_T(f) \mid A\right] \leq \psi\left( \log \left(1 + \frac{1}{P[A]}\right) \right),
\end{align}
which, from the first paragraph of the proof, implies the wished claim.
\end{proof}

\section{Proof of the excess risk bounds for ISWERM}

\subsection{Proof of \cref{thm:ISWERMslow}}

\begin{proof}[Proof of \cref{thm:ISWERMslow}]
Let 
\begin{align}
    M_T(f) := \frac{1}{T} \sum_{t=1}^T (P_{g_t} - \delta_{O_t}) (\ell(f) - \ell(f_1)).
\end{align}
Since $\widehat{R}_T(\widehat{f}_T) - \widehat{R}_T(f_1) \leq 0$, and from the diameter assumption \ref{asm:diameter}, we have that $R^*(\widehat{f}_T) - R^*(f_1) \leq \sup \{M_T(f) : f \in \F, \|\ell(f) - \ell(f_1) \|_{2,g^*} \leq \rho_0 \| \Lambda \|_{2,g^*}\}$. Therefore, from the diameter assumption (\cref{asm:diameter}), \cref{thm:max_ineq_IS_weighted_MEP} yields, via the change of variable $r = \rho \|\Lambda\|_{2,g^*}$, for any $x > 0$, $\rho^- \in [0, \rho_0/2]$, that it holds with probability at least $1 - 2 e^{-x}$ that
\begin{align}
    R^*(\widehat{f}_T) - R^*(f_1) \leq \|\Lambda\|_{2,g^*} &\left\lbrace \rho^- +\sqrt{\frac{\gamma^{\text{avg}}_T}{T}} \int_{\rho^-}^{\rho_0} \sqrt{\log (1 + \Nb(\epsilon \|\Lambda\|_{2,g^*}, \ell(\F), \|\cdot\|_{2,g^*}} d \epsilon \right. \\
    &+ \frac{b_0 \gamma_T^{\max}}{T} \log (1 + \Nb(\rho_0 \|\Lambda\|_{2,g^*}, \ell(\F), \|\cdot\|_{2,g^*}))\\
    &\left. + \sqrt{\frac{\gamma^{\text{avg}}_T x}{T}} + \frac{b_0 \gamma_T^{\max} x}{T} \right\rbrace.
\end{align}
In the case $p \in (0,2)$, setting $\rho^- = 0$ and $x = \log(1 / \delta)$, and plugging in the entropy assumption (\cref{asm:entropy}) immediately yield the claim. In the case $p > 2$, setting $x= \log(1 / \delta)$, plugging in the entropy assumption and optimizing the value of $\rho^-$ yields the claim.
\end{proof}

\subsection{Proof of \cref{thm:ISWERMfast}}

\begin{proof}
From convexity of $f \mapsto \ell(f, \cdot)$ and of $\F$, the following implication holds, for any $r > 0$:
\begin{align}
    & \exists f \in \F,\ R^*(f) - R^*(f_1) \geq r^2 \qquad \text{and} \widehat{R}_T(f) - \widehat{R}_T(f_1) \leq 0\\
    \implies &\exists f \in \F,\ R^*(f) - R^*(f_1) = r^2 \qquad \text{and} \widehat{R}_T(f) - \widehat{R}_T(f_1) \leq 0.
\end{align}
Let 
\begin{align}
    M_T(f) := \frac{1}{T} \sum_{t=1}^T (P_{g_t} - \delta_{O_t}) (\ell(f) - \ell(f_1)).
\end{align}
Let $\rho > 0$. Since $\widehat{R}_T(\widehat{f}_T) - \widehat{R}_T(f_1) \leq 0$, we have that
\begin{align}
    & P \left[ R^*(\widehat{f}_T) - R^*(f_1) \geq \rho^2 \|\Lambda\|_{2,g^*} \|_{2,g^*} \right] \\
    \leq & P \left[ \exists f \in \F, R^*(f) - R^*(f_1) = \rho^2 \| \Lambda \|_{2,g^*}  \text{ and } \widehat{R}_T(f) - \widehat{R}_T(f_1) \leq 0 \right] \\
    \leq & P \left[ \exists f \in \F, \sup \left\lbrace M_T(f) : f \in \F, \|\ell(f) - \ell(f_1) \|_{2,g^*}  \right\rbrace \lesssim \|\Lambda \|_{2,g^*} \rho^\alpha \right],
\end{align}
where we have used the variance bound in the last line (\cref{asm:variance}). From theorem \ref{thm:max_ineq_IS_weighted_MEP} and the loss diameters assumption \ref{asm:diameter}, we have, for any $x, \rho > 0$, that it holds with probability at least $1 - 2 e^{-x}$ that
\begin{align}
    \sup \left\lbrace M_T(f) : f \in \F, \|\ell(f) - \ell(f_1) \|_{2,g^*}  \right\rbrace \lesssim \psi_T(\rho),
\end{align}
with 
\begin{align}
    \psi_T(\rho) := \| \Lambda \|_{2,g^*} &\left\lbrace \sqrt{\frac{\gamma^\text{avg}_T}{T}}  \int_0^{\rho^\alpha} \sqrt{ \log ( 1 + \Nb(\epsilon \|\Lambda \|_{2,g^*}, \ell(\F), \|\cdot\|_{2,g^*}))} d \epsilon \right. \\
    &+ \frac{b_0 \gamma^{\max}_T}{T} \log (1 + \Nb(\rho^\alpha \|\Lambda \|_{2,g^*}, \ell(\F), \|\cdot\|_{2,g^*}) \\
    &+ \left. \rho^\alpha \sqrt{\frac{\gamma^{\text{avg}}_T x}{T}} +  \frac{b_0 \gamma^{\max}_T x}{T} \right\rbrace.
\end{align}
Therefore, if $\rho$ is such that $\rho^2 \|\Lambda \|_{2,g^*} \geq \psi_T(\rho)$, then with probability at least $1 - 2 e^{-x}$, 
\begin{align}
    R^*(\widehat{f}_T) - R^*(f_1) \lesssim \rho^2 \|\Lambda \|_{2,g^*}.
\end{align}
We therefore compute an upper bound on $\psi_T(\rho)$. From the entropy assumption (\ref{asm:entropy}), we have that
\begin{align}
    \psi_T(\rho) \lesssim \|\Lambda \|_{2,g^*} \left\lbrace \sqrt{\frac{\gamma^{\text{avg}}_T}{T}} \rho^{\alpha(1-p/2)} + \frac{b_0 \gamma^{\max}_T}{T} \rho^{-p \alpha} + \rho^\alpha \sqrt{\frac{\gamma^{\text{avg}}_Tx}{T}} + \frac{b_0 \gamma^{\max}_T}{T} x\right\rbrace.
\end{align}
Therefore, a sufficient condition for $\rho^2 \|\Lambda \|_{2,g^*} \geq \psi_T(\rho)$ is that 
\begin{align}
    \rho^2 \geq \max \left\lbrace \sqrt{\frac{\gamma^{\text{avg}}_T}{T}} \rho^{\alpha(1-p/2)},\frac{b_0 \gamma^{\max}_T}{T} \rho^{-p \alpha},\rho^\alpha \sqrt{\frac{\gamma^{\text{avg}}_Tx}{T}},\frac{b_0 \gamma^{\max}_T}{T} x\right\rbrace
\end{align}
that is
\begin{align}
    \rho^2 \geq \max \left\lbrace \left(\frac{\gamma^{\text{avg}}_T}{T} \right)^{\frac{1}{2-\alpha + p \alpha /2}}, \left(\frac{\gamma^{\text{avg}}_T}{T} \right)^{\frac{1}{2-\alpha}}, \left( \frac{b_0 \gamma^{\max}_T}{T}\right)^{\frac{1}{1 + p \alpha /2}}, \frac{b_0 \gamma^{\max}_T x}{T}\right\rbrace,
\end{align}
which immediately implies the wished claim.
\end{proof}

\section{Proof of the results on least squares regression using ISWERM}

\begin{proof}[Proof of lemma \ref{lemma:properties_square_loss}]
For any $o = (x,a,y) \in \cO$, $f, f': \cO \to \Rl$, we have 
\begin{align}
    | \ell(f)(o) - \ell(f')(o) | =& |2y - f(a,x) - f'(a,x)||f(a,x) - f_1(a,x)| \\
    \leq & 4 \sqrt{M} | f(a,x) - f_1(a,x)|,
\end{align}
which is the second claim.
This inequality further gives that, for any $f \in \F$
\begin{align}
    \left\lVert \ell(f) - \ell(f_1) \right\rVert_{2,g^*} \leq 4 \sqrt{M} \left\lVert f - f_1\right\rVert_{2,g^*}.
\end{align}
We now show that $\|f - f_1 \|_{2,g^*} \leq R^*(f) - R^*(f_1)$. Recall the definition of $\mu$: for any $(a,x) \in \A, \X$, $\mu(a,x) := E_{p_Y}[Y \mid A=a, X=x]$. From Pythagoras, $R^*(f) = E[(Y - \mu(A,X))^2] + \|\mu - f\|_{2,g^*}^2$. For any $h_1, h_2: \A \times \X \to \Rl$, denote $\langle h_1, h_2 \rangle := E_{p_X,g^*}[h_1(A,X) h_2(A,X)]$. We have that
\begin{align}
    & R^*(f) - R^*(f_1) - \left\lVert f - f_1 \right\rVert_{2,g^*} \\
    =& \left\lVert f - \mu \right\rVert_{2,g^*}^2 - \left\lVert f_1 - \mu \right\rVert_{2,g^*}^2 - \left\lVert f-f_1\right\rVert_{2,g^*}^2 \\
    =& \langle f - f_1, f_1 - \mu \rangle \\
    \geq & 0,
\end{align}
since $f_1$ is the projection for $\langle \cdot,\cdot \rangle$ of $\mu$ onto the convex set $\F$. This yields the first claim.
\end{proof}

\begin{proof}[Proof of \cref{thm:iswls}]
From the definition of the range of the outcome and of the regression functions, $o \mapsto \sqrt{M}$ is an envelope for $\F$ and $o \mapsto 4 M$ is an envelope for $\ell(\F)$. From \cref{lemma:bkting_preservation_long_version}, and the fact that $\ell$ is $4 \sqrt{M}$-equiLipschitz w.r.t. its first argument, 
\begin{align}
    \Nb(4 M \epsilon, \ell(\F), \|\cdot\|_{2,g^*}) \lesssim \Nb(\sqrt{M} \epsilon, \F, \|\cdot\|_{2,g^*}) \lesssim \epsilon^{-p},
\end{align}
where the last inequality follows from the fact that \cref{asm:entropy2} holds for $\F$ with envelope $o \mapsto \sqrt{M}$. Therefore, \cref{asm:entropy} holds for envelope $\Lambda : o \mapsto 4 M$. In addition, for this envelop definition, \cref{asm:diameter} holds with $\rho_0 = b_0 = 1$. Finally, from lemma \cref{lemma:properties_square_loss},
\begin{align}
    \left\lVert \ell(f) - \ell(f_1) \right\rVert_{2,g^*} \leq 4 \sqrt{M} ( R^*(f) - R^*(f_1))^{1/2} \\
    = 2 (4 M) \left( \frac{ R^*(f) - R^*(f_1) }{4 M} \right)^{\frac{1}{2}},
\end{align}
that is \cref{asm:variance} holds.
\cref{thm:iswls} then follows directly by instantiating \cref{thm:ISWERMslow} and \cref{thm:ISWERMfast}, respectively in the case $p> 2$ and in the case $p \in (0,2)$, with $\Lambda : o \mapsto 4 M$, $\alpha = 1$, $b_0 = \rho_0 = 1$.
\end{proof}

\section{Proof of the results on policy learning using ISWERM}

\begin{proof}[Proof of \cref{thm:policyslow} and \cref{thm:policyfast}]
Note that since the outcome has range $[-M,M]$, $\ell$ is $M$-equiLipschitz w.r.t. its first argument. Therefore, from \cref{lemma:bkting_preservation_long_version} and the fact that $\F$ satisfies \cref{asm:entropy2} with envelope constant equal to 1, \cref{asm:entropy} holds with envelope $\Lambda : o \mapsto M$.

Furthermore, \cref{asm:diameter} holds for $b_0 = \rho_0 = 1$. Therefore, instantiating \cref{thm:ISWERMslow} with $\Lambda =M$, $\rho_0 = b_0 = 1$ yields \cref{thm:policyslow}. 

Under realizability and \cref{asm:margin}, \cref{lemma:vb_from_margin_cond}, gives us that \cref{asm:variance} holds for $\alpha = \nu / (\nu + 1)$. \cref{thm:policyfast} follows by instantiating \cref{thm:policyfast} with $\alpha= \nu / (\nu + 1)$, $b_0 = \rho_0 = 1$.
\end{proof}

\section{Technical lemmas}

\subsection{Long version of lemma 4 in \cite{bibaut2019fast}}

We restate here under our notation the full version of lemma 4 in \cite{bibaut2019fast}, of which we gave a short version under the form of lemma \ref{lemma:bkting_unimodal_loss}. 

\begin{lemma}[Lemma 4 in \cite{bibaut2019fast}, long version]\label{lemma:bkting_preservation_long_version}
Let $\ell: \mathcal{F} \times \mathcal{O} \to \mathbb{R}$. Suppose that there exists $\widetilde{\ell}: \mathbb{R} \times \mathcal{O} \to \mathbb{R}$ such that
\begin{itemize}
    \item it holds that $\forall f:\cO \to \mathbb{R}, o \in \cO, \ell(f, o) = \widetilde{\ell}(f(o), o)$,
    \item $\widetilde{\ell}$ is $L$-equiLipschitz w.r.t. its first argument, that is, \begin{align}
        |\widetilde{\ell}(z_2,o) - \widetilde{\ell}(z_1,o)| \leq L|z_1 - z_2|, \forall o \in \cO, z_1, z_2 \in \mathbb{R}
    \end{align}
    \item for every $o \in \cO$, $z \mapsto \widetilde{\ell}(z, o)$ is unimodal.
\end{itemize}
Then, for any measure $\mu$ on $\cO$, any $p \geq 1$, and $\epsilon > 0$, it holds that
\begin{align}
    \Nb(L \epsilon, \ell(\F), \|\cdot\|_{\mu, p}) \leq \Nb(\epsilon, \F, \|\cdot\|_{\mu, p}).
\end{align}
\end{lemma}

\subsection{Proof of the variance bound under margin condition}

\begin{proof}[Proof of \cref{lemma:vb_from_margin_cond}]
By assumption there exists $f_1\in\F$ such that $R^*(f_1)=\E_{p_X}\mu^*(X)$. Applying \cref{asm:margin} with $u=0$ shows that we necessarily have $\abs{\argmin_{a\in\A}\mu(X,a)}=1$ almost surely. Therefore, almost surely, $f_1(X,a^*(X))=1$ and $f_1(X,a)=0$ for $a\neq a^*(X)$.

Now fix any $f\in\F$. Given $X$, let $A\in\A$ be random variable draw from $f(X,\cdot)$. We will henceforth denote expectations and probabilities as wrt $(X,A)\sim p_X\times f$. For brevity we will also denote $A^*=a^*(X)$.
Note that
$$
\|\ell(f,\cdot)-\ell(f_1,\cdot)\|_{2,g^*}^2\leq
M^2\Prb{A^*\neq A}
$$
and that 
$$
\|\Lambda\|_{2,g^*}^2
\left(\frac{R^*(f)-R^*(f_1)}{\|\Lambda\|_{2,g^*}}\right)^{\alpha}=M^2(\Eb{\mu(X,A)-\mu(X,A^*)}/M)^{\nu/(\nu+1)}.
$$
Denoting $\Delta=\min_{a\in\A\backslash\{a^*(X)\}}\mu(X,a)-\mu^*(X)$, \cref{asm:margin} says that for some $\kappa>0$ we have $\Prb{\Delta\leq u}\leq (\kappa u/M)^\nu$, where $1^\infty=1$ and $x^\infty=0$ for $x\in[0,1)$.

Fix $u>0$. Then
\begin{align}
\Eb{\mu(X,A)-\mu(X,A^*)}
&=\Eb{(\mu(X,A)-\mu(X,A^*))\Ind(A\neq A^*)}\\
&\geq\Eb{(\mu(X,A)-\mu(X,A^*))\Ind(A\neq A^*,\Delta>u)}\\
&\geq u\Prb{A\neq A^*,\Delta>u}\\
&= u\prns{\Prb{A\neq A^*}-\Prb{A\neq A^*,\Delta\leq u}}\\
&\geq u\prns{\Prb{A\neq A^*}-\Prb{\Delta\leq u}}\\
&\geq u\prns{\Prb{A\neq A^*}-(\kappa u/M)^\nu}.
\end{align}
Set $u=((\nu+1)\kappa/M)^{-1/\nu}\Prb{A\neq A^*}^{1/\nu}$ and obtain
$$
\Eb{\mu(X,A)-\mu(X,A^*)}\geq \nu (\nu+1)^{-(\nu+1)/\nu} (\kappa/M)^{-1}\Prb{A\neq A^*}^{(\nu+1)/\nu},
$$
whence
$$
\Prb{A\neq A^*}\leq \nu^{-\nu/(\nu+1)} (\nu+1) \prns{(\kappa/M)\Eb{\mu(X,A)-\mu(X,A^*)}}^{\nu/(\nu+1)}.
$$
We conclude that
$$
\|\ell(f,\cdot)-\ell(f_1,\cdot)\|_{2,g^*}^2\lesssim
M^2 \prns{\Eb{\mu(X,A)-\mu(X,A^*)}/M}^{\nu/(\nu+1)}
$$
as desired.
\end{proof}

\section{Additional Details and Results for the Empirical Investigation}

Here we provide additional details and results for \cref{sec:empirics}.

\subsection{Contextual Bandit Data from Multi-Class Classification Datasets}\label{sec:dataconstruction}

To construct our data, we turn $K$-class classification tasks into a $K$-armed contextual bandit problems \citep{dudik2014doubly,dimakopoulou2017estimation,su2019cab}, which has the benefits of reproducibility using public datasets and being able to make uncontroversial comparisons using actual ground truth data with counterfactuals. We use the public OpenML Curated Classification benchmarking suite 2018 (OpenML-CC18; BSD 3-Clause license) \citep{bischl2017openml}, which has datasets that
vary in domain, number of observations, number of classes and number of features. Among these, we select the classification datasets which have less than 60 features. This results in 51 classification datasets from OpenML-CC18 used for evaluation. 
\cref{tab:openml} summarizes the characteristics of the 51 OpenML datasets used.

\begin{table}[h!]
\centering
\begin{tabular}{|c|c|}
\hline
Samples & Count \\
\hline
$< 1000$  & 16  \\ 
\hline
$\geq 1000$ and $< 10000$ & 25  \\ 
\hline
$\geq 10000$ & 10 \\
\hline
\end{tabular}
\hspace{5pt}
\begin{tabular}{|c|c|}
\hline
Classes & Count \\
\hline  
$= 2$ & 30 \\
\hline
$> 2 \text{ and } < 10$  & 15  \\ 
\hline
$ \geq 10 $ & 6  \\
\hline
\end{tabular}
\hspace{5pt}
\begin{tabular}{|c|c|}
\hline
Features & Count \\
\hline  
$\geq 2 \text{ and } < 10$ & 14 \\
\hline
$\geq 10 \text{ and } < 30$  & 22  \\ 
\hline
$\geq 30 \text{ and } \leq 60$ & 14  \\
\hline
\end{tabular}
\vspace{5pt}
\caption{Characteristics of the 51 OpenML-CC18 datasets used for evaluation.}
\label{tab:openml}
\vspace{-\baselineskip}\end{table}

Each dataset is a collection of pairs of covariates $X$ and labels $L\in\{1,\dots,K\}$. We transform each dataset to the contextual bandit problem as follows. At each round, we draw $X_t,L_t$ uniformly at random with replacement from the dataset. We reveal the context $X_t$ to the agent, and given an arm pull $A_t$, we draw and return the reward $Y_t \sim \mathcal{N}(\textbf{1}\{A_t = L_t\}, 1)$. To generate our data, we set $T=100000$ and use the following $\epsilon$-greedy procedure. We pull arms uniformly at random until each arm has been pulled at least once. Then at each subsequent round $t$, we fit $\widehat{\mu}_{t-1}$ using the data up to that time.
Specifically, for each $a$, we take the data $\{(X_s,Y_s):1 \leq s \leq t-1, A_s = a\}$ and pass it to a regression algorithm in order to construct $\widehat{\mu}_{t-1}(\cdot,a)$. 
In \cref{sec:empirics}, we presented results where we use \verb|sklearn|'s \verb|LinearRegression| to fit $\widehat{\mu}_{t-1}(\cdot,a)$ (using \verb|sklearn| defaults).
In \cref{apdx:empirics}, we repeat the experiments where we instead use \verb|sklearn|'s \verb|DecisionTreeRegressor| (using \verb|sklearn| defaults).
We set $\tilde A_t(x)=\argmax_{a=1,\dots,K}\widehat{\mu}_{t-1}(a,x)$ and $\epsilon_t=t^{-1/3}$. We then let $g_t(a\mid x)=\epsilon_t/K$ for $a\neq \tilde A_t(x)$ and $g_t(\tilde A_t(x)\mid x)=1-\epsilon_t+\epsilon_t/K$. That is, with probability $\epsilon_t$ we pull a random arm, and otherwise we pull $\tilde A_t(X_t)$.

\subsection{Additional Results}
\label{apdx:empirics}

In \cref{sec:empirics}, we presented results where we use a linear-contextual $\epsilon$-greedy bandit algorithm to collect the data.
Here, we repeat our experiments when the data are instead collected by a tree-contextual $\epsilon$-greedy bandit algorithm, as described in \cref{sec:dataconstruction} above.
The results are shown in \cref{fig:mse-treebandit}. The conclusions are generally the same: ISWERM compares favorably for fitting linear models, while all methods perform similarly for fitting tree models.


\begin{figure}
    \centering
\subfloat[LASSO outcome model with cross-validated regularization parameter.]{\includegraphics[width=1\linewidth]{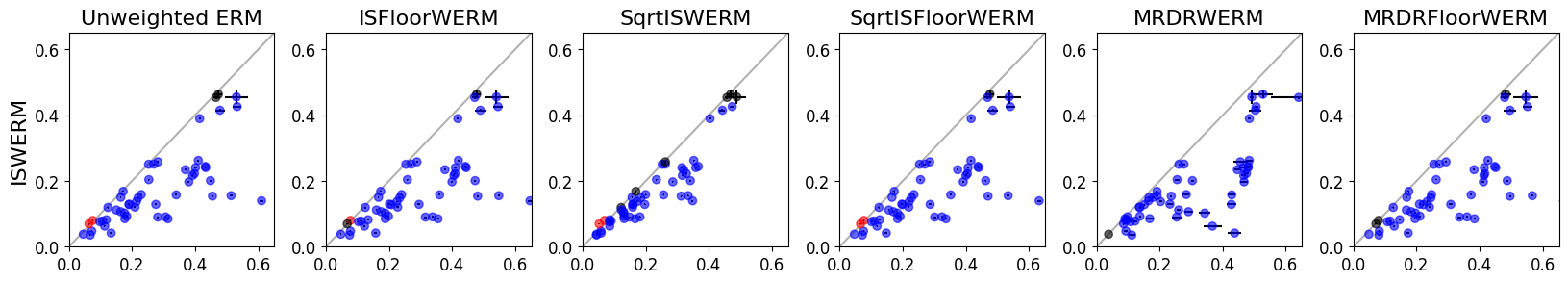}} \\
\subfloat[Ridge outcome model with cross-validated regularization parameter.]{\includegraphics[width=1\linewidth]{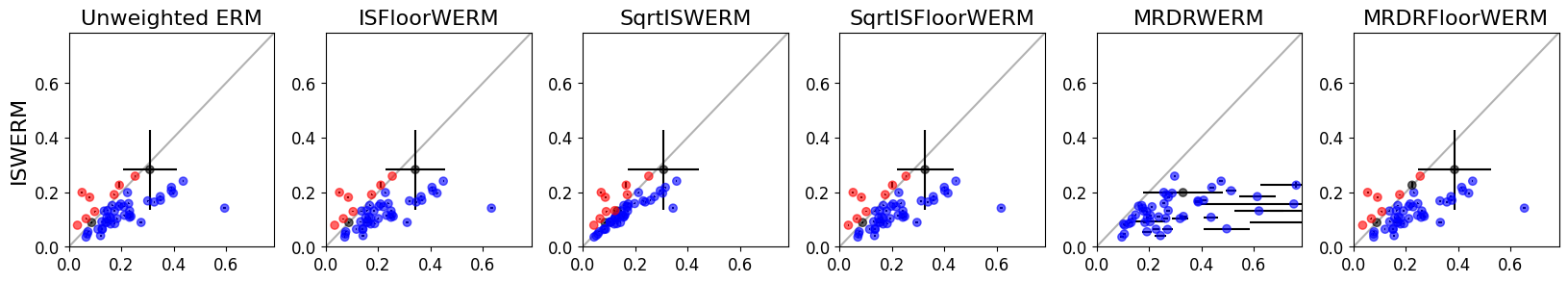}} \\
\subfloat[CART outcome model with unrestricted tree depth.]{
\includegraphics[width=1\linewidth]{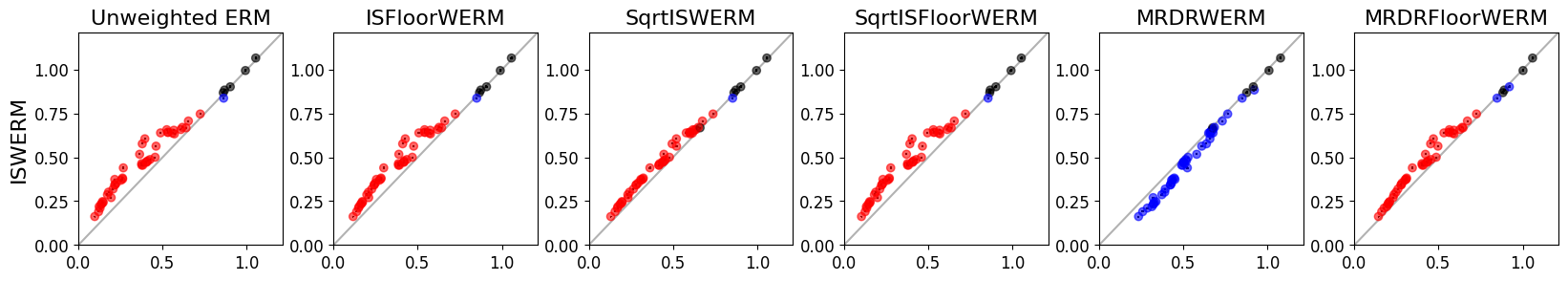}}
  \caption{Comparison of weighted regression run on contextual-bandit-collected data. Each dot is one of 51 OpenML-CC18 datasets. Lines denote $\pm1$ standard error. Dots are blue when ISWERM is clearly better, red when clearly worse, and black when indistinguishable within one standard error.}
    \label{fig:mse-treebandit}
\end{figure}

\subsection{Code and Execution Details}
\label{code}
The IPython notebook to reproduce the experimental results of the main paper and the appendix is included as an attachment in the Supplemental Material. One needs to obtain an OpenML API key to run this code (instructions can be found at https://docs.openml.org/Python-guide/) and replace the string \verb|'YOURKEY'| in \verb|summarize_openmlcc18()| and in \verb|download_openmlcc18()| functions with it.
After that, if the notebook is executed as is, it reproduces Figure \ref{fig:mse-linearbandit} (38h 26min on a single Intel Xeon machine with 32 physical cores/64 CPUs). Changing variable  \verb|bandit_model| from \verb|'linear'| to \verb|'tree'| reproduces Figure \ref{fig:mse-treebandit} (56h 45min on a single Intel Xeon machine with 32 physical cores/64 CPUs).
\end{document}